\documentclass[letterpaper]{article} 
\usepackage[draft]{aaai2026}  
\usepackage{times}  
\usepackage{helvet}  
\usepackage{courier}  
\usepackage[hyphens]{url}  
\usepackage{graphicx} 
\usepackage{natbib}  
\usepackage{caption} 
\usepackage{lineno}
\usepackage{latexsym}
\usepackage{amssymb, amsmath, amsthm}
\usepackage{booktabs}
\usepackage{enumitem}
\usepackage{graphicx}
\usepackage[dvipsnames]{xcolor}

\usepackage{todonotes}
\setuptodonotes{inline}

\usepackage{xspace}
\usepackage{subcaption}

\usepackage[ruled,vlined]{algorithm2e}

\usepackage{xcolor}
\usepackage{tikz}
\usetikzlibrary{shapes.geometric}
\usetikzlibrary{arrows.meta,arrows}
\usetikzlibrary{automata}

\usepackage[most]{tcolorbox} 
\usepackage{fontawesome5}

\usepackage{multirow}
\usepackage[T1]{fontenc}   
\usepackage{caption} 

\usepackage{placeins}

\usepackage{times}
\usepackage{helvet}
\usepackage{courier}

\definecolor{green_hl}{rgb}{0.0, 0.56, 0.0}

\newcommand{\BibTeX}{B\kern-.05em{\sc i\kern-.025em b}\kern-.08em\TeX}

\newcommand{\LTL}{{\sc ltl}\xspace}
\newcommand{\LTLf}{{\sc ltl}$_f$\xspace}
\newcommand{\LDLf}{{\sc ldl}$_f$\xspace}
\newcommand{\QLTLf}{{\sc ltl}$_f$\ensuremath{[\mathcal{F}]}\xspace}
\newcommand{\QATLf}{{\sc atl}$_f$\ensuremath{[\mathcal{F}]}\xspace}

\newcommand{\true}{\ensuremath{\top}}  
\newcommand{\false}{\ensuremath{\bot}} 

\newcommand{\R}{\mathcal{R}}

\newcommand{\Until}{\ensuremath{\mathcal{U}}\xspace} 

\newcommand{\Next}{X}

\newcommand{\Always}{G}
\newcommand{\Eventually}{F}

\newcommand{\Release}{\ensuremath{\mathop{\R}}}

\newcommand{\eqdef}{\stackrel{def}{=}}
\renewcommand{\iff}{\ensuremath{\leftrightarrow}}

\newcommand{\StatesSet}{\ensuremath{Q}}
\newcommand{\InitState}{\ensuremath{q_0}}
\newcommand{\RegisterSet}{\ensuremath{\mathcal{V}}}
\newcommand{\TransitionFn}{\ensuremath{\delta_q}}
\newcommand{\RegisterFn}{\ensuremath{\delta_v}}
\newcommand{\RewardFn}{\ensuremath{\delta_r}}

\newcommand{\ExtendedMDP}{\ensuremath{\mathcal{M} \otimes \mathcal{A}}}

\newcommand{\registerAtI}{\ensuremath{\mathrm{val}_{\mathcal{A}_\phi}(\,\varphi,\, i\,)}\xspace}

\newcommand{\traceAtI}{\ensuremath{\bigl[\![\varphi,i]\!\bigr](\lambda)}\xspace}
\newcommand{\eveAlwTrue}{\ensuremath{\Eventually \Always \, \top }\xspace}

\newcommand{\gym}{\textsc{gym}\xspace}

\newtheorem{theorem}{Theorem}
\newtheorem{lemma}[theorem]{Lemma}

\newtheorem{definition}{Definition}
\newtheorem{example}{Example}

\newtcolorbox{monitorBox}[3][]{
  enhanced,                    
  equal height group=#1,       
  title=#2,                    
  colbacktitle=#3!40,          
  coltitle=black,              
  colframe=#3!75!black,        
  colback=gray!10,             
  fonttitle=\bfseries,         
  center title,
  title filled,                
  arc=3mm,
  boxsep=5pt
}

\title{Expressive Temporal Specifications for Reward Monitoring}
\author {
    Omar Adalat\textsuperscript{\rm 1},
    Francesco Belardinelli\textsuperscript{\rm 1}
}
\affiliations {
    \textsuperscript{\rm 1}Imperial College London\\
    o.adalat24@imperial.ac.uk, francesco.belardinelli@imperial.ac.uk
}

\newif\ifwithappendix
\withappendixtrue

\begin{document}

\maketitle

\begin{abstract}
Specifying informative and dense reward functions remains a pivotal challenge in Reinforcement Learning, as it directly affects the efficiency of agent training. In this work, we harness the expressive power of quantitative Linear Temporal Logic on finite traces (\QLTLf) to synthesize \emph{reward monitors} that generate a dense stream of rewards for \emph{runtime}-observable state trajectories. By providing nuanced feedback during training, these monitors guide agents toward optimal behaviour and help mitigate the well-known issue of \emph{sparse rewards} under long-horizon decision-making, which arises under the Boolean semantics dominating the current literature. Our framework is algorithm-agnostic and only relies on a state labelling function, and naturally accommodates specifying non-Markovian properties. Empirical results show that our quantitative monitors consistently subsume and, depending on the environment, outperform Boolean monitors in maximising a quantitative measure of task completion and in reducing convergence time.
\end{abstract}


\begin{links}
    \link{Code}{https://github.com/nightly/quantitative-reward-monitoring}
\end{links}

\section{Introduction}
\emph{Reinforcement Learning} (RL) offers a powerful framework for sequential decision-making under uncertainty, leveraging sampling-based methods to iteratively refine policies to optimality \cite{sutton1998reinforcement}. Central to the success of any RL algorithm is the careful design of its reward function
\cite{eschmann2021reward}.
Ill-defined reward functions might lead to unintended behaviours \cite{leike2017ai}, unforeseen consequences \cite{denison2024sycophancy}, and catastrophic failures \cite{festor2022assuring}. As such, the specification of reward functions is key to the development of safe and robust RL, including the AI alignment problem \cite{christian2021alignment}.

Within this landscape, temporal logics \cite{BaierKatoen} stand out as an elegant alternative to conventional reward specification approaches.
By offering a formal, high-level language for expressing complex objectives and constraints, temporal logics enable more structured reward design and facilitate clearer reasoning about an agent's goals and behaviour.

\begingroup
\renewcommand\thefootnote{}%
\footnotetext{*~This paper is the extended version (with appendix) of the paper in the Proceedings of the 40th AAAI Conference on Artificial Intelligence (AAAI-26).}%
\addtocounter{footnote}{-1}%
\endgroup

Consequently, temporal logic-based reward definitions can reduce the burden of manual engineering, while improving interpretability and maintainability, thus paving the way for more robust RL systems \cite{li2019formal, hasanbeig2020deep}.

In the current literature, logical specifications in RL (including reward specifications) have mainly settled on qualitative, Boolean semantics, where a formula is evaluated as true or false \cite{sadigh2014learning}. However, Boolean semantics inherently leads to \emph{sparse rewards}, as feedback is often provided only upon 
complete satisfaction of the relevant formula. Sparse rewards pose significant challenges to RL algorithms, as they hinder effective credit assignment and increase sample complexity \cite{andrychowicz2017hindsight}, limiting performance in long-horizon tasks.

To mitigate these issues, we investigate the use of quantitative Linear Temporal Logic on finite traces (denoted as \QLTLf) to specify and synthesize {\em reward monitors} for providing richer feedback. Quantitative semantics enable assessing the degree to which a given formula is satisfied besides being just true or false.

\paragraph{Contributions.}
In this paper we provide a procedure to automatically construct \emph{quantitative reward monitors} starting with rewards specified as formulas in \QLTLf.
Such monitors provide a reward in real-time by consuming a trace of states. The agent is ascribed a corresponding reward based on progress, using real-valued evaluation of a formula in \emph{quantitative linear temporal logic interpreted on finite traces}.
We evaluated the proposed approach empirically, investigating the efficacy of the approach in terms of convergence speed and a domain-specific task completion measure, against Boolean monitors and handcrafted reward functions. 
We demonstrate superior performance compared to Boolean monitors (in settings where quantitative information can be exploited) and, in some cases, manually-specified reward functions. Furthermore, we show how a quantitative reward monitor can be constructed with linear overhead in the number of transitions and states w.r.t.~the formula size, and how non-Markovian properties are captured through temporally extended goals. Since quantitative semantics provide denser rewards, this leads to greater sample efficiency, ability to handle long-horizon tasks, and avoid sparse rewards.

\paragraph{Scheme of the paper.} 
Section~\ref{sec:background} provides the prerequisite knowledge and foundation on reward monitors. Section~\ref{section:qrm} introduces quantitative reward monitors and their construction algorithm. Section~\ref{section:experimental_eval} provides our empirical results with methodology and discussion. Section~\ref{sec:related_work} discusses related work; and finally, Section~\ref{sec:conclusion} provides our conclusions and future work.

\section{Background}
\label{sec:background}

In this section we introduce the background required to build quantitative monitors from \QLTLf specifications. 

\textbf{Reinforcement Learning} (RL) environments are typically modelled as \emph{Markov decision processes} (MDPs), defined as tuples $\mathcal{M} = (\mathcal{S}, A, \Pr, \mathcal{R}, \gamma)$,
where $\mathcal{S}$ is a  set of {\em states}; $A$ is a set of {\em actions}; $\Pr: \mathcal{S} \times A \to \Delta(\mathcal{S})$ is the {\em transition kernel}, with $\Delta(\mathcal{S})$ being the set of probability distributions over $\mathcal{S}$, $\mathcal{R} : \mathcal{S} \times A \times \mathcal{S} \to \mathbb{R}$ is the (Markovian) {\em  reward function},  and $\gamma \in [0,1]$ is the {\em discount factor}. 

In an MDP, a (stochastic) {\em policy} is a function $\pi : \mathcal{S} \to \Delta(A)$ and  its performance is measured by the expected discounted return
$
V^{\pi}(s)\;=\;\mathbb{E}_{\pi}\!\Bigl[\;\sum_{t=0}^{\infty}\gamma^{t}\mathcal{R}_{t+1}\;\Bigm|\;\mathcal{S}_{0}=s\Bigr],
$
where the expectation is over the trajectory generated by following policy $\pi$ from the initial state $\mathcal{S}_{0}=s$.  
The goal of RL is to find an optimal policy $\pi^{\star}$ satisfying
$\pi^{\star}\in\arg\max_{\pi}V^{\pi}(s)$ for every $s\in\mathcal{S}$. In this setting, the \textit{Markov property} ensures that the next state and reward depend only on the current state and action.  A \emph{labelled Markov decision process} extends an MDP by adding a labelling function $\mathcal{L}: \mathcal{S} \to [0,1]^{\mathcal{P}}$ from the set $\mathcal{S}$ of states to the set $\mathcal{P}$ of atoms.

\paragraph{Quantitative Linear Temporal Logic on finite traces.}
Linear Temporal Logic on finite traces (\LTLf) shares the syntax with Linear Temporal Logic (\LTL), but has different semantics, due to being interpreted on traces of finite length \cite{de2013linear}. Naturally, episodes in RL have finite length, which is suitable for specifications interpreted on finite traces. Quantitative \LTLf (\QLTLf) shares the same syntax as \LTLf, but instead of having a two-valued semantics, computations are assigned a real-valued scalar within the range $[0, 1]$, thus extending Zadeh's fuzzy propositional logic with linear temporal connectives \cite{lamine2000using,faella2008model}.

\begin{definition}[\QLTLf syntax]
    Given a set $\mathcal{P}$ of atoms, the syntax of \QLTLf is defined as follows, where $p\in \mathcal{P}$:
$$
\varphi::= 
\top \mid 
p \mid \neg \varphi\mid \varphi \wedge \varphi \mid 
\Next  \varphi \mid
\varphi \: \Until \: \varphi \mid \varphi \Release \varphi
$$
where $\top$ denotes true and $\bot \eqdef \lnot \top$ denotes false.
\end{definition}

Here $\Next$ (Next),  $\Until$ (Until), and $\Release$ (Release) are the standard \LTL operators.
The Boolean connectives, $\lor$ (or), $\to$ (implies), and $\leftrightarrow$ (iff) can be introduced by their standard abbreviations.
The temporal operators $\Eventually$ (eventually) and $\Always$ (always) can be derived as $\Eventually \varphi \eqdef \top \: \Until \: \varphi$ and $\Always \varphi \eqdef \false \Release \varphi$. 

We also consider a fragment of \LTLf, termed \emph{safe-\LTLf}, to syntactically identify whether a given specification is purely a safety objective \cite{sistla1994safety}.
\begin{definition}[Safe-\LTLf] \label{def:safe_ltlf}
A safe-\LTLf formula $\varphi$ is an until-free formula in negation normal form, such that its syntax comprises:
\begin{eqnarray*}
  \varphi & ::= & 
    \top \mid
    p \mid \lnot p \mid
    \varphi_1 \land \varphi_2 \mid
    \varphi_1 \lor \varphi_2 \mid
    \Next \varphi \mid
    \varphi_1 \Release \varphi_2
\end{eqnarray*}
\end{definition}
The semantics for \QLTLf is provided as follows, where $\lambda$ denotes a finite trace of states, and $\lambda_i$ represents the $i^{\text {th }}$ state in the trace.
%
\begin{definition}[Semantics of \QLTLf] \label{def:semantics}
  The evaluation $[\![\varphi, i]\!](\lambda) \in [0,1]$ of formula $\varphi$ on the finite trace $\lambda$ 
  is defined inductively as follows: 
$$
  \begin{aligned}
  [\![\true, i]\!](\lambda) &= 1\\
  [\![p, i]\!](\lambda) &= \mathcal{L}(\lambda_{i})(p) \\
  [\![\lnot \varphi, i]\!](\lambda) &= 1 - [\![\varphi, i]\!](\lambda)\\
  [\![\varphi_{1}\wedge \varphi_{2}, i]\!](\lambda) &= 
     \min\bigl([\![\varphi_{1}, i]\!](\lambda),\,[\![\varphi_{2}, i]\!](\lambda)\bigr)\\
  %
 [\![\Next \varphi, i]\!](\lambda)
 &=
 \begin{cases}
   [\![\varphi, i+1]\!](\lambda), & \text{if } i < |\lambda|\\
   0, & \text{otherwise},
 \end{cases}\\[6pt]
  [\![\varphi_{1} \: \Until \: \varphi_{2}, i]\!](\lambda)
  &=
  \max_{\,i \le j \le |\lambda|}
    \min_{i \le k < j} \Bigl(
      [\![\varphi_{2}, j]\!](\lambda), [\![\varphi_{1}, k]\!](\lambda)
    \Bigr) \\
    [\![\varphi_{1}\, \Release \,\varphi_{2}, i]\!](\lambda)
    &=
    \min_{i \le j \le |\lambda|} 
      \max_{i \le k < j} \Bigl( [\![\varphi_{2}, j]\!](\lambda),\, [\![\varphi_{1}, k]\!](\lambda) \Bigr)
     \\[6pt]
  \end{aligned} 
$$
\end{definition}


The truth value for derived logical connectives and temporal operators can be obtained from Def.~\ref{def:semantics}.

\begin{figure*}[h]
    \centering
    \begin{subfigure}[b]{0.48\textwidth}
        \centering
        \definecolor{lightpurple}{rgb}{0.8, 0.8, 1}
\definecolor{darkpurple}{rgb}{0.6, 0.6, 0.9}

\tikzset{
    mynode/.style={circle, draw=darkpurple, fill=lightpurple, thick, minimum size=24pt},
    myedge/.style={->, >=stealth', thick},
    mylabel/.style={midway},
    initial glow/.style={
        mynode,
        preaction={
          draw=darkpurple,
          line width=3pt,
          opacity=.50
        }
   },
}

\begin{tikzpicture}[auto]

  \node[mynode, initial glow] (0) at (0,2) {0};
  \node[mynode] (1) at (-2,0) {1};
  \node[mynode, double=darkpurple, double distance=1pt] (2) at (2,0) {2};


  \draw[myedge] (0) edge[loop above] node[mylabel, above] {$\lnot a$} ();
  
  \draw[myedge] (0) -- node[mylabel, above left] {$a \land \lnot b$} (1);
  \draw[myedge] (0) -- node[mylabel, above right] {$a \land b$} (2);
  
  \draw[myedge] (1) -- node[mylabel, below] {$b$} (2);

  \draw[myedge] (1) edge[loop left] node[mylabel, left] {$\lnot b$} ();
  \draw[myedge] (2) edge[loop right] node[mylabel, right] {$\true$} ();

\end{tikzpicture}
        \caption{Boolean reward monitor (Moore machine) for the \LTLf formula $\lnot a \: \Until \: (a \land \Eventually b)$, where state 0 is the initial state, and state 2 can be viewed as the accepting state which outputs a reward.}
    \end{subfigure}\hfill
    \begin{subfigure}[b]{0.48\textwidth}
        \centering
        \definecolor{lightpurple}{rgb}{0.8, 0.8, 1}
\definecolor{darkpurple}{rgb}{0.6, 0.6, 0.9}

\tikzset{
    mynode/.style={circle, draw=darkpurple, fill=lightpurple, thick, minimum size=24pt},
    myedge/.style={->, >=stealth', thick},
    mylabel/.style={midway},
    initial glow/.style={
        mynode,
        preaction={
          draw=darkpurple,
          line width=3pt,
          opacity=.50
        }
   },
}

\begin{tikzpicture}[auto]

    \node[mynode, initial glow] (0) at (-2, 2) {0};
    \node[mynode] (1) at (0, 2) {1};
    \node[mynode] (2) at (2, 2) {2};

    \draw[myedge] (0) -- (1);
    \draw[myedge, bend left=18]  (1) to (2);
    \draw[myedge, bend left=18]  (2) to (1);

    \node[draw=darkpurple, fill=lightpurple!50, thick,
          rounded corners, align=left, inner sep=0pt,
          text width=1\linewidth, font=\small] (regBox2)
          at (0,4) {
            {\centering \qquad \qquad \: \: \textbf{Register updates ($\delta_v$) in State 2:}} \\
        
        $a\gets\mathcal{L}_a(s')$;
        $\lnot a \gets 1-\RegisterSet(a);\;$
        $\min(\lnot a)\gets\min\{\min(\lnot a),\lnot a\};\;$\\
          
        $b\gets\mathcal{L}_b(s');\;$
        $\max(b)\gets\max\{\Eventually(b),b\}$\;\\
        
        $\max(a\land\max(b))\gets
            \max\{\max(a\land\max(b)),a\land\max(b)\}$\\


        \hspace{4cm} $\cdots$
    };

    \draw[->, thick, darkpurple] (regBox2) -- (2);

\end{tikzpicture}
        \caption{Quantitative reward monitor for the \QLTLf formula $\lnot a \: \Until \: (a \land \Eventually b)$, where state 0 is the distinguished initial state. The register update function is partially displayed only for state 2 for brevity. }
    \end{subfigure}
\vspace{0.3cm}
\caption{A boolean monitor with \LTLf semantics (left-hand side) and quantitative monitor with \QLTLf semantics (right-hand side) constructed for the same formula: $\lnot a \: \Until \: (a \land \Eventually b)$.}
\label{fig:bool_and_quant_monitors}
\end{figure*}

\textbf{Reward monitors} are a type of transducer, which consumes a trace $\tau \eqdef \bigl(s_1, s_2, \ldots, s_n \bigr)$ of states and provides as output a sequence $r = ( r_1, r_2, \ldots, r_n )$ of scalar rewards, one for each time step. We define a Boolean reward monitor as a specific type of transducer: 
\begin{definition}[BRM]
   A {\em Boolean reward monitor} is 
   a Moore machine, 
   defined as a tuple $\mathcal{B} = (Q, q_0, \Sigma, \Gamma, \delta, \theta)$, where $Q$ is the set of the states, with initial state $q_0$. $\Sigma$ is the input alphabet, $\Gamma$ is the output alphabet, 
   $\delta : Q \times \Sigma \to Q$ is the transition function, $\theta: Q \to \mathbb{R}$ is the output (reward) function.
\end{definition}

Construction of a reward monitor is given by a set of specification-reward pairs, which are defined as follows:
\begin{definition}[Specification-reward pairs]
    A specification-reward pair is a tuple $(\varphi, \rho)$, where $\varphi$ is an \LTLf or \QLTLf formula,  and $\rho$ 
    is the associated scalar weight.
\end{definition}

In order to construct a Boolean reward monitor from an \LTLf formula, one can transform the formula into a deterministic finite automaton (DFA) \cite{de2013linear}, and then use it as a reward machine returning a reward of 0 on all states, except for accepting states outputting a reward of 1 (which can be multiplied by scalar $\rho$).
\section{Quantitative Reward Monitors}
\label{section:qrm}
In this section we provide the definition of quantitative reward monitors, as well as the procedure to build them from \QLTLf specifications.

\subsection{Monitor Definition}

A {\em quantitative reward monitor} is a 
\emph{finite state machine with registers}, defined as follows: 
\begin{definition}[QRM]
    A {\em Quantitative Reward Monitor} is defined as a tuple $\mathcal{A} = (\StatesSet, q_0, \TransitionFn, \RegisterSet, \RegisterFn, \RewardFn)$, where $\StatesSet$ is the set of states, with $q_0 \in \StatesSet$ the initial state, a set $\mathcal{V} \subseteq T \times \mathbb{R}$ of register-value pairs  (we write $\mathcal{V}(t)$ to retrieve the stored register value $t$) where $T$ is the set of registers, $\delta_q: \StatesSet \to \StatesSet$ is the transition function, $\delta_v: \StatesSet \times T \to \RegisterSet$ is the register update function, and $\delta_r = t_{reward} \times \rho$ is the reward function, where $t_{reward}$ is the reward register and $\rho$ is the scalar in the reward specification pair.
\end{definition}

The registers update throughout the trace, and track aspects relevant to formula evaluation, for example the minimum and maximum of a subformula as required by the semantics of temporal operators $\Eventually$ and $\Always$ respectively. Intuitively, the reward register $t_{reward}$ denotes the degree of satisfaction of a given  specification's formula, hence providing quantitative feedback. In Figure \ref{fig:bool_and_quant_monitors} we demonstrate how both Boolean and quantitative monitors would be constructed for the same formula $\lnot a \: \Until \: (a \land \Eventually b)$.

\subsection{Safety Specifications}
Reward monitors (both Boolean and quantitative) focus on two principal classes of objectives, namely \emph{reachability} and \emph{safety} objectives.  From $\varphi$ in a reward-specification pair, we can syntactically infer whether the specification comprises a safety objective 
(according to 
Def.~\ref{def:safe_ltlf}), where we denote a safety formula (syntactically inferred) by $\varphi_{safety}$.

A \emph{safety specification} requires that once $\varphi_{safety}$ is violated at any time step, the monitor emits a fixed penalty for the remainder of the trace; that is, rewards are clamped to $\zeta \in \mathbb{R}$ with $\zeta \le 0$ (typically $\zeta=0$ or a negative constant). This covers both terminating failures, where the episode ends immediately, as well as non-terminating hazards, where the agent should not be able to recover to a positive return after an unsafe act (e.g. a self-driving car running a red light).  Formally, if a safety formula $\varphi_{\text{safety}}$ is violated at index $i$,   then $\delta_r$ outputs $\zeta$ for all subsequent time indices $t\ge i$. Temporal credit assignment enables propagating the penalty back to the unsafe action.

\subsection{Monitor Synthesis}
\label{sec:rm_construction}

Function $\textsc{construct}$ for constructing a reward monitor $\mathcal{A}_{\varphi}$ is defined inductively on the \QLTLf formula $\varphi$,  with the support of the memoization function $\textsc{synth}$. As many sub-monitors are created recursively, storing them becomes useful to avoid repetitive computation, which is the purpose of Algorithm \ref{algo:synth} \textsc{synth}. 

Let $\mathcal{L}_p(s)$ denote the value of the atomic proposition $p$ in state $s$, where $s \in \mathcal{S}$ of the labelled MDP, the function $\textsc{construct}$ works as follows:
\begin{itemize}
   \item For $\varphi = \true$, the monitor is assigned states $\StatesSet \gets \{q_0\}$, set of register-value pairs $\mathcal{V} \gets \bigl\{ (t_{\top}, 1) \bigr\}$, transition function $\TransitionFn \gets \bigl\{(q_0, q_0) \bigr\}$, register update function $\delta_v \gets \emptyset$, and reward function $\delta_r \gets \mathcal{V}(t_\top)$.
    
    \item For $\varphi = p$ for $p \in \mathcal{P}$, the monitor is assigned states $\StatesSet \gets \{q_0, q_1\}$, set of register-value pairs $\mathcal{V} \gets \bigl\{(t_p, 0)\bigr\}$, transition function $\TransitionFn \gets \bigl\{ (q_0, q_1), (q_1, q_1) \bigr\}$, register update function $\RegisterFn(q_0, t_p) \gets \mathcal{L}_p(s)$, and reward function $\delta_r \gets \mathcal{V}(t_p)$. 
    
    \item For negation $\varphi = \lnot \psi$, let us create $\mathcal{A}^{\psi}$ from procedure $\textsc{Synth}$, where $\StatesSet \gets Q^{\psi}$, $q_0 \gets q_0^{\psi}$, $\mathcal{V} \gets \mathcal{V}^{\psi} \cup \bigl\{ (t_{\lnot \psi}, 1 - \mathcal{V}^{\psi}(t_{\psi})) \bigr\}$, $\TransitionFn = \TransitionFn^{\psi}$, $\RegisterFn \gets \delta_v^{\psi}$, and then $\delta_r \gets \mathcal{V}(t_{\lnot \psi})$. Finally, for all states $q \in \StatesSet$, assign $\delta_v(q, t_{\lnot \psi}) = 1 - \mathcal{V}(t_{\psi})$.
    
    \item For conjunction $\varphi = \varphi_1 \land \varphi_2$, let us create $\mathcal{A}^{\varphi_1}$ and $\mathcal{A}^{\varphi_2}$ from procedure $\textsc{Synth}$. Then we can take the Cartesian product of the two monitors, such that $\StatesSet = \StatesSet^{\varphi_1} \times \StatesSet^{\varphi_2} = \{(q^{\varphi_1}, q^{\varphi_2}) \mid q^{\varphi_1} \in \StatesSet^{\varphi_1},\, q^{\varphi_2} \in \StatesSet^{\varphi_2}\}$, $q_0 \gets (q_0^{\varphi_1}, q_0^{\varphi_2})$, and $\mathcal{V} \gets \mathcal{V}^{\varphi_1} \cup \mathcal{V}^{\varphi_2} \cup \{ (t_{\varphi_1 \land \varphi_2}, \min(\mathcal{V}(t_{\varphi_1}), \mathcal{V}(t_{\varphi_2})))\}$, and $\TransitionFn: \StatesSet^{\varphi_1} \times \StatesSet^{\varphi_2} \to \StatesSet^{\varphi_1} \times \StatesSet^{\varphi_2}$. Then, for all states $q \in \StatesSet$, we can assign $\delta_v \gets \delta_v^{\varphi_1} \cup \delta_v^{\varphi_2}$ and $\delta_v(q, t_{\varphi_1 \land \varphi_2}) \gets \min(\mathcal{V}(t_{\varphi_1}), \mathcal{V}(t_{\varphi_2}))$. Finally, let $\delta_r \gets \mathcal{V}(t_{\varphi_1 \land \varphi_2})$.

    \item For temporal next $\varphi = \Next \psi$, the procedure is identical to an atomic proposition $p$, but the labelling of $s'$ (as a result of taking action $a$ in state $s$) is considered instead of $s$.
    
    \item For temporal until $\varphi = \varphi_1 \: \Until \: \varphi_2$, let us start by obtaining $\mathcal{A}^{\varphi_1}$ and $\mathcal{A}^{\varphi_2}$ from procedure $\textsc{Synth}$. Let $q_{-1}$ and $q_{-2}$ denote the final state of a monitor and the penultimate state of a monitor respectively. We form a new monitor $\mathcal{A}^{\varphi}$ as follows, comprising a state space of $\StatesSet = \StatesSet^{\varphi_1} \times \StatesSet^{\varphi_2} = \{(q^{\varphi_1}, q^{\varphi_2}) \mid q^{\varphi_1} \in \StatesSet^{\varphi_1},\, q^{\varphi_2} \in \StatesSet^{\varphi_2}\}$, and $q_0 \;\gets\; (q_0^{\varphi_1}, \; q_0^{\varphi_2})$. Then, we can state
    $\mathcal{V} \gets \mathcal{V}^{\varphi_1} \: \cup \: \mathcal{V}^{\varphi_2} \cup \bigl\{ (t_{\min \varphi_1}, 1)\bigr\} \cup \bigl\{(t_{\max \varphi_2}, 0) \bigr\} \cup \bigl\{ (t_{\varphi_1 \: \Until \: \varphi_2}, 0) \bigr\}$. Then, for all states $q \in \StatesSet$, we can assign $\delta_v(q, t_{\min \varphi_1}) \gets \min(\mathcal{V}(t_{\min \varphi_1}), \mathcal{V}(t_{\varphi_1}))$, $\delta_v(q, t_{\max \varphi_2}) \gets \max(\mathcal{V}(t_{\max \varphi_2}), \mathcal{V}(t_{\varphi_2}))$, and $\delta_v(q, t_{\varphi_1 \: \Until \: \varphi_2}) \gets \max(\mathcal{V}(t_{\varphi_1 \: \Until \: \varphi_2}), \min(\mathcal{V}(t_{\min \varphi_1}), \mathcal{V}(t_{\max \varphi_2})))$.      To support continuous updates, we assign $\TransitionFn(q_{-1}) \gets q_{-2}$, $\RegisterFn(q_{-2}, t) \gets \RegisterFn(q_{-1}, t) \: \forall t$, and all contained atomic registers are renamed (to avoid name clashing), and at initialisation are filled with the initial state labels $\mathcal{L}_p(\mathcal{S}_0)$, updatable with $\delta_v(q_0, t_p) \to \mathcal{L}_p(s')$. Finally, we can assign $\RewardFn \gets \mathcal{V}(t_{\varphi_1 \: \Until \: \varphi_2})$.

    \item For temporal release $\varphi = \varphi_1 \Release \varphi_2$, the same procedure as until $\Until$ can be followed, substituting any instances of $\min$ with $\max$ and vice versa.
\end{itemize}
Monitors for temporal operators $\Eventually$ (Eventually) and $\Always$  (Always) can also be constructed primitively, with fewer registers as an implementation optimisation, only needing to track the max and min resp. of the associated subformula.
Further details of optimisations and abbreviated connectives are provided in Appendix~\ref{sec:appendix_quantitative_construction}.
\begin{theorem}
  \label{theorem:linear}
  The state and transition overhead of quantitative reward monitor construction from a \QLTLf formula $\varphi$ is linear with respect to the size of $\varphi$.
\end{theorem}

Theorem ~\ref{theorem:linear}'s proof is provided in Appendix~\ref{proof_linearity}.
Note that for \LTLf a monitor using alternating automata can also be constructed in linear time \cite{de2015synthesis}.
\begin{algorithm}[h]
  \DontPrintSemicolon
  \SetAlgoLined
  \caption{\textsc{Synth}. Monitor construction with memoization for an \QLTLf formula.}
  
  \KwIn{\QLTLf formula $\varphi$, Cache }
  \KwOut{A reward monitor $\mathcal{A} = (\StatesSet, \InitState, \TransitionFn, \RegisterSet, \RegisterFn, \RewardFn)$}
  
  \If{$\varphi$ \textbf{is} in Cache}{
  \Return Cache[$\varphi$]\;
  }
  
  $\mathcal{A} \gets \textsc{construct}(\varphi)$\tcp*[l]{Build inductively}
  
  Cache[$\varphi$] $\gets \mathcal{A}$
  
  \Return $\mathcal{A}$\;
  \label{algo:synth}
\end{algorithm}

\begin{lemma}
    The procedure \textsc{synth} returns a quantitative reward monitor for the provided \QLTLf formula. 
    \label{lemma:formula_monitor}
\end{lemma}

The proof of Lemma~\ref{lemma:formula_monitor} appears in the Appendix~\ref{sec:lemma_proof}
\begin{theorem}[Correctness]
  \label{theorem:correctness}
  Let $\varphi$ be an \QLTLf formula, $\mathcal{A}_{\varphi}$ the QRM constructed from $\varphi$ as per Alg.~\ref{algo:synth}, and $\lambda$ a finite trace of length $n$. Let \registerAtI denote the value stored in the reward register at time index $i$. Then, for each index $1 \leq i \leq n$,
$$
\registerAtI = \traceAtI 
$$
\end{theorem}
We provide the details of the proof of Theorem  \ref{theorem:correctness} in the Appendix~\ref{sec:correctness_proof}.

\subsection{Composition and Learning}

Learning involves taking the product between the quantitative monitor $\mathcal{A}$ and the MDP of the environment $\mathcal{M}$. We can define this product as the extended MDP  $\ExtendedMDP$
as follows:
\begin{definition}[Extended MDP] Let $\mathcal{M} = (\mathcal{S}, A, \Pr, \mathcal{R}, \gamma, \mathcal{L})$ be a labelled MDP  and $\mathcal{A} = (\StatesSet, \InitState, \TransitionFn, \RegisterSet, \RegisterFn, \RewardFn)$ a quantitative reward monitor. Their synchronous product $\mathcal{M} \otimes \mathcal{A} = (\mathcal{S}', A, \Pr',\mathcal{R}', \gamma,  \mathcal{L})$ is  defined as follows, where $q \in Q$, $a \in A$, and $s \in \mathcal{S}$:
\begin{itemize}
    \item State-space $\mathcal{S}' \eqdef Q \times \mathcal{S}$
    \item Transition kernel: $\Pr' \bigl((q,s),a,(q',s')\bigr) \eqdef$
    $$
          \begin{cases}
            \Pr(s,a,s') & \text{if } \TransitionFn\!\bigl((q, s) \bigr) = q',\\[4pt]
            0           & \text{otherwise.}
          \end{cases}
    $$ 
    \item Reward function: $\mathcal{R}'((q, s), a, (q', s')) \eqdef \RewardFn(\RegisterSet' = \RegisterFn((q', s'), \mathcal{L}(s')), \rho)$
    \item Remaining elements are inherited unchanged.
\end{itemize}
\end{definition}
The extended state-space including $q \in Q$ permits specifying non-Markovian (temporally extended) goals. However, a Markovian policy suffices when considering the product $\ExtendedMDP$, as stated in Theorem \ref{theorem:markovian_policy}.
\begin{theorem}
    \label{theorem:markovian_policy}
    A Markovian policy $\pi$ in MDP $\ExtendedMDP$ suffices to optimally capture the non-Markovian goals encoded by the monitor. 
\end{theorem}
A proof is provided in the Appendix~\ref{sec:markovian_policy_proof}.

\begin{table*}[ht]
  \centering

  \resizebox{\textwidth}{!}{
  \begin{tabular}{ll ccc ccc ccc}
      \toprule
       & 
       & \multicolumn{3}{c}{\textbf{Mean Episode Number}}
       & \multicolumn{3}{c}{\textbf{Mean Time (seconds)}}
       & \multicolumn{3}{c}{\textbf{Task Completion (\%, mean $\pm$ 95\% CI)}} \\
      \cmidrule(lr){3-5}\cmidrule(lr){6-8}\cmidrule(lr){9-11}
       & 
       & \textbf{Base} & \textbf{Boolean} & \textbf{Quant.}
       & \textbf{Base} & \textbf{Boolean} & \textbf{Quant.}
       & \textbf{Base} & \textbf{Boolean} & \textbf{Quant.} \\
      \midrule
      \multirow{4}{*}{\textbf{Classic}}
        & Acrobot        & \textbf{None} & 600* & None & \textbf{None} & 133.95* & None & \textbf{99.07\% $\pm 0.09\%$} & 4.84\% $\pm 0.40\%$ & 94.77\% $\pm 0.30\%$ \\
        & Cartpole       & \textbf{None} & None &  None & \textbf{None} & None & None & \textbf{44.67\% $\pm 0.42\%$} & 30.92\% $\pm 0.35\%$ & 33.34\% $\pm 0.37\%$ \\
        & Mountain Car   & 600* & 600* & \textbf{None} & 41.46* & 42.98* & \textbf{None} & 39.82\% $\pm 0.07\%$ & 37.04\% $\pm 0.06\%$ & \textbf{42.28\% $\pm 0.07\%$} \\
        & Pendulum       & \textbf{None} & 600* & None & \textbf{None} & 181.17* & None & \textbf{74.39\% $\pm 0.27\%$} & 45.03\% $\pm 0.18\%$ & 72.11\% $\pm 0.28\%$ \\
      \cmidrule(lr){1-11}
      \multirow{3}{*}{\textbf{Toy}}
        & Frozen Lake      & 216* & \textbf{45} & 29* & 0.04* & \textbf{0.03} & 0.01* & 58.86\% $\pm 0.10\%$ & \textbf{62.16\% $\pm 0.10\%$} & 58.96\% $\pm 0.10\%$ \\
        & Cliff Walking    & \textbf{21} & 25 & 2.3 & \textbf{0.01} & 0.04 & 0.003 & \textbf{85.33\% $\pm 0.05\%$} & 84.38\% $\pm 0.05\%$ & 84.95\% $\pm 0.05\%$ \\
        & Taxi             & \textbf{11} & 1301* & 1221* & \textbf{0.01} & 13.89* & 4.26 & \textbf{71.19\% $\pm 0.07\%$} & 46.50\% $\pm 0.03\%$ & 52.76\% $\pm 0.03\%$ \\
      \cmidrule(lr){1-11}
      \multirow{3}{*}{\textbf{Box2D}}
        & Bipedal Walker    & \textbf{None} & None & None & \textbf{None} & None & None & \textbf{16.78\% $\pm 0.54\%$} & 7.43\% $\pm 0.15\%$ & 14.43\% $\pm 0.41\%$  \\
        & Lunar Lander      & \textbf{None} & None & None & \textbf{None} & None & None & \textbf{57.45\% $\pm 0.36\%$} & 43.44\% $\pm 0.17\%$ & 45.61\% $\pm 0.16\%$ \\
      \cmidrule(lr){1-11}
      \multirow{3}{*}{\textbf{Safety Gridworlds}}
        & Island Navigation   & \textbf{1302} & 71* & 144* & \textbf{0.43} & 0.05* & 0.10* & \textbf{96.81\% $\pm 0.02\%$} & 76.37\% $\pm 0.03\%$ & 76.72\% $\pm 0.03\%$ \\
        & Conveyor Belt    & \textbf{191}* & 468* & 518* & \textbf{0.15} & 0.7* & 0.7* & \textbf{60.96\% $\pm 0.10\%$} & 27.1\% $\pm 0.09\%$ & 28.70\% $\pm 0.09\%$ \\
        & Sokoban     & 1129* & 71* & \textbf{103} & 0.47* & 0.11* & \textbf{0.14} & 50.23\% $\pm 0.01\%$ & 52.00\% $\pm 0.01\%$ & \textbf{53.15\% $\pm 0.02\%$} \\
      \bottomrule
  \end{tabular}}
  \caption{Convergence data for all environments, split by base reward function, Boolean and quantitative monitors. Task completion
  entries are mean $\pm$ 95\% confidence interval (half-width). (*) denotes suboptimal (comparatively) policy convergence. }
  \label{table:cw}
\end{table*}

\section{Experimental Evaluation}
\label{section:experimental_eval}
In this paper, we consider the notion of \emph{reward convergence} with respect to the different reward mechanisms: Boolean monitors, quantitative monitors, and the pre-defined reward function, as provided in Definition~\ref{def:reward_convergence}, measuring reward stability to check convergence and agent performance \cite{dulac2021challenges,machado2018revisiting,brockman2016openai}.

\begin{definition}[Reward Convergence] \label{def:reward_convergence}
Given episodic rewards $\{\mathcal{R}_t\}_{t=1}^T$, define the exponentially moving average $E_t$ as
$$
E_t \eqdef \beta E_{t-1}+(1-\beta)\mathcal{R}_t,\qquad \beta \eqdef 1-\tfrac{2}{N+1},
$$
with span $N$. Define checkpoints every $N$ steps: $C_i \eqdef E_{iN}$.

An RL run \emph{converges (reward-wise)} if, given a tolerance $\tau$, for the last $P$ checkpoint pairs,
\[
\lvert C_i - C_{i-1}\rvert \le \tau \quad \text{for all such } i.
\]
\end{definition}

To objectively measure learning performance, one method would be tracking \emph{convergence}, although this can be hindered by the fact that an agent may still converge to a sub-optimal policy, particularly when using a non-informative reward function. Additionally, it is also the case that cumulative rewards themselves are not as useful as a metric when the quantitative monitor and handcrafted function (from the original environment implementation) are able to produce a reward at each time step (as contrasted to the Boolean monitor), and at arbitrary magnitudes of intermediate rewards. 

Therefore, we use the notion of \emph{task completion}, which is an evaluative (unobservable for the agent) \emph{performance function} run at the terminal time point of the episode, and is universally suitable and uniform for all reward producers. This is similar to the performance function used in \cite{leike2017ai}.
\begin{definition}[Task completion]
    A performance function which assigns a scalar in the range of $[0, 1]$, computed at the \emph{terminal time step} (whether the episode ends by success, termination, or truncation). This signal is hidden from the agent during training, and is only used for evaluation.
\end{definition}

As for the scalars $\rho$ in the specification-reward pair, we select them to be the same in the Boolean and quantitative monitors in cases where the formulas are equivalent.

\begin{figure*}[t]
  \centering
  \begin{subfigure}{1.0\textwidth}
    \centering
    \includegraphics[width=\linewidth]{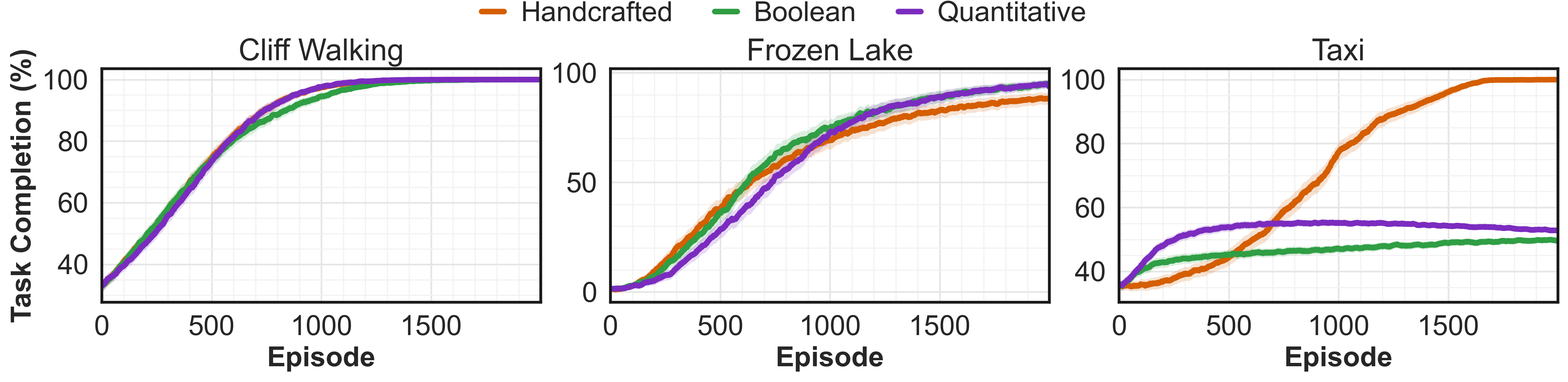}
    \label{fig:toy_graphs}
  \end{subfigure}
  \begin{subfigure}{1.0\textwidth}
    \centering
    \includegraphics[width=\linewidth]{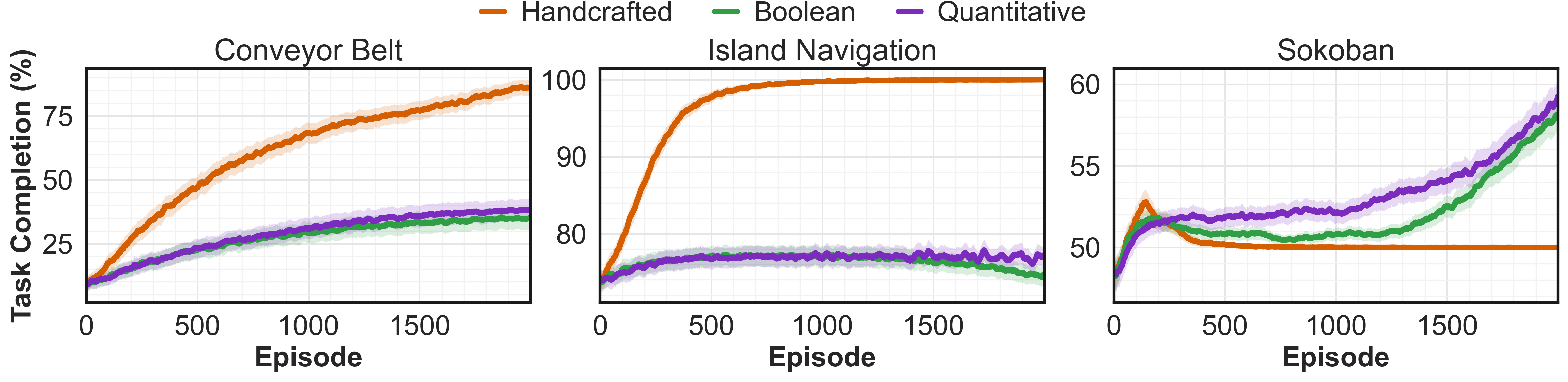}
    \label{fig:safety_gridworld_graphs}
  \end{subfigure}
  \caption{Task-completion percentages across \textsc{toy} and \textsc{safety-gridworlds} environments, $\pm$ 95\% confidence interval. Curves were smoothed using a moving average with a window size of $21$.}
  \label{fig:all_graphs}
\end{figure*}

\subsection{Environments} \label{subsec:environments}
Empirical results were gathered for our quantitative reward monitor across varying \textsc{Gymnasium} environments, specifically using the following sets of environments:
\begin{itemize}
    \item \textsc{Classic}: low-dimensional continuous-control benchmarks based upon classical control theory problems \cite{sutton1995generalization,6313077,moore1990efficient,aastrom2000swinging} that test stabilisation and swing-up behaviour, consisting of: \emph{Acrobot}, \emph{Cartpole}, \emph{Mountain Car}, and \emph{Pendulum}.
    \item \textsc{Toy}: small grid-world domains with finite state-action spaces that permit exact dynamic-programming or tabular RL evaluation \cite{dietterich2000hierarchical,brockman2016openai,sutton1998reinforcement}, which illustrate path-planning under severe penalties as well as hierarchical goal decomposition. We considered the following environments: \emph{Cliff Walking}, \emph{Frozen Lake}, and \emph{Taxi}.
    \item \textsc{Box2D}: continuous-control problems with contact dynamics simulated in the Box2D engine \cite{brockman2016openai}, requiring balance, locomotion, or soft landing.  We consider \emph{Bipedal Walker} and \emph{Lunar Lander}.
    \item \textsc{Safety Gridworlds}: a set of 2D gridworld environments provided by DeepMind \cite{leike2017ai}, that evaluate the safety of RL agents against undesirable behaviours such as unsafe exploration, safe interruptibility, irreversible side effects, reward gaming, and robustness to self-modification and adversaries. Our experiments consist of the following environments:  \emph{Conveyor Belt}, \emph{Island Navigation}, and \emph{Sokoban}.
\end{itemize}

For each of these environments, we compare the manually specified reward function (e.g. provided by \textsc{Gymnasium}), a Boolean reward monitor, and a quantitative reward monitor. The complete descriptions and specification of all environments including all fluent definitions (both Boolean and quantitative) against the observation and action trajectories is detailed in the Appendix~\ref{sec:appendix_env_spec_details}.

\subsubsection{Specifying quantitative properties}
For each environment specified, we need to specify Boolean and quantitative properties for our labelled MDP. We also make use of a quantitative measure of task completion. Let us consider the environment of \textsc{Cartpole} in Example \ref{example:cartpole_env}.

\begin{example}[Cartpole's complete environmental specification]
\label{example:cartpole_env}
  
  The relevant reward-specification pairs for \textsc{Cartpole} can be defined as the following for both the Boolean and quantitative settings: $\bigl( (\Eventually(\Always(reach\_goal)), 2), \: (\Always(balanced), 4) \bigr)$, with the former expressing a persistence property and the latter a safety property. For the first specification, we want to ensure that we always remain in an intended goal position (not slipping backwards or forwards, here we can assume $goal\_pos = 2$). The second specification requires always remaining balanced, which is useful regardless of how far the agent is from the goal position. 

  Below is the mathematical form of definitions for the labels used by the labelling function $\mathcal{L}(s)$ of the labelled MDP, or in other words the quantitative properties:
  \begin{eqnarray*}
      \textit{balanced} & = &
      \begin{cases} 
      \frac{0.209 - |angle|}{0.209}, & \text{if } |angle| \leq 0.209 \\ 
      0, & \text{else}
  \end{cases}\\
  %
      \textit{reach\_goal} & = &
      \begin{cases} 
      \frac{current\_pos}{goal\_pos}, & \text{if } current\_pos > 0.01 \\ 
      0, & \text{else}
      \end{cases}
  \end{eqnarray*}
  
  In the case of Boolean atoms, $balanced$ returns $\true$ iff $angle \leq 0.209$, else $\false$. As for $reach\_goal$, it is true if the agent is within $0.1$ of the $goal\_pos$. The task completion measure (performance function) is specified as an equally weighted mean of $reach\_goal$ and $balanced$.
\end{example}

\subsection{Reinforcement Learning Algorithms}
The following section briefly describes the algorithms we use, including both tabular and policy gradient settings. \emph{Tabular Q-learning} is used with \textsc{Toy} and \textsc{Safety Gridworlds} environments, due to faring well with the relatively small discrete state-action spaces. \emph{Proximal policy optimization} is used with \textsc{Classic} and \textsc{Box2D} environments to adequately handle continuous state and action spaces.

\textbf{Q-learning} \cite{watkins1992q} is a value-based method that iteratively updates the action-value function $Q(s,a)$ according to the Bellman optimality equation:
\[
Q(s,a) \leftarrow Q(s,a) + \alpha \bigl(r + \gamma \max_{a'} Q(s', a') - Q(s,a)\bigr),
\]
where $\alpha$ is the learning rate, $\gamma$ is the discount factor, $r$ is the reward observed, $(s,a)$ denotes the current state-action pair, and $(s',a')$ denotes the next state-action pair. 

\textbf{Proximal policy optimization} (PPO)
\cite{schulman2017proximal} is a policy-gradient method designed for continuous or high-dimensional state and action spaces. Instead of directly estimating action-values in a table, PPO maintains a parameterized policy $\pi_{\theta}(a|s)$ and optimizes it by maximizing a clipped surrogate objective:
\[
L^{\text{CLIP}}(\theta) = \hat{\mathbb{E}}_{t} \left[ \min(r_t(\theta)\hat{A}_t, 
\text{clip}(r_t(\theta), 1 - \epsilon, 1 + \epsilon)\hat{A}_t) \right],
\]
where $r_t(\theta) = \frac{\pi_{\theta}(a_t|s_t)}{\pi_{\theta_\text{old}}(a_t|s_t)}$ is the probability ratio between the new and old policies, $\hat{A}_t$ is the advantage function estimate at time step $t$, and $\epsilon$ is a hyperparameter that bounds the ratio, helping to prevent destructive policy updates.  Due to its on-policy nature, PPO is relatively stable. Its clipping mechanism ensures that updates do not move the policy too far from the current parameters, which helps mitigate high-variance updates and promotes more reliable convergence.

\subsection{Experimental Results}
Table \ref{table:cw} presents our results in terms of reward convergence time and episode number, as well as our evaluation metric of quantitative task completion.  The mean convergence time and episode number are taken from averaging a number of runs, with each run containing a fixed amount of training episode iterations. Likewise, the task completion metric is averaged over a number of runs, for each episode across the learning, where the terminal time step outputs the task completion progress. 
The number of episodes and runs used depends on the environment, where the number of episodes is shown in the graphs in Figure~\ref{fig:all_graphs}, 
and the remaining hyperparameters are specified in Appendix~\ref{sec:appendix_hyperparams}.

Across environments, the quantitative monitor consistently matches or outperforms the Boolean monitor, and sometimes surpasses the manually specified reward function. Within simple and small gridworld-like environments, it was often simpler to use Boolean signals, rather than distance metrics such as Manhattan distance, as they do not take into account obstacles (thus encouraging exploring around non-optimal areas). Hence, the performance difference between Boolean and quantitative monitors is more pronounced in \textsc{Classic} and \textsc{Box2D} environments. We restricted ourselves to defining a quantitative specification using only the information known to the agent (observations and actions), which is not always the case where manually-specified reward functions are used. Our approach is easily extensible to incorporating such environmental signals by extending the scope of the labelling function. 
The learning performance gap of the quantitative and Boolean monitors is entirely determined by how well a quantitative measure can be defined per environment using the state trajectory, which is domain-dependent. However, in cases where writing a quantitative specification proves difficult, the performance should always be at least equal to using a Boolean monitor.
\section{Related Work}
\label{sec:related_work}

\textbf{Reward engineering} is the practice of designing and tuning the reward structure to better align with the intended learning outcomes. Certain approaches exist specifically to counter \emph{sparse} and \emph{delayed} rewards, which remain compatible with QRMs and are useful for environments where feedback can only be provided towards the end of an episode. \emph{Reward redistribution} techniques such as Align-RUDDER \cite{patil2022align} shift terminal returns onto a few decisive actions, densifying feedback while preserving optimal policies. Causal variants \cite{zhang2023interpretable} relax alignment, but are susceptible to weak confounding. \emph{Hindsight Experience Replay} (HER) \cite{andrychowicz2017hindsight} retroactively relabels failed roll-outs with achieved goals, generating synthetic dense rewards, though goal selection becomes intractable in large or partially observed spaces and may bias value estimates. \emph{Reward shaping} augments each step-return with $F(s_t,s_{t+1})=\gamma\Phi(s_{t+1})-\Phi(s_t)$, guaranteeing policy invariance \cite{ng1999policy}. Potentials $\Phi$ are either handcrafted (e.g., distance-to-goal) or learned \cite{grzes2010online}, which must be informative yet low-variance.

\paragraph{Temporal logic rewards.}  A survey of reinforcement learning with temporal logic rewards is presented in \cite{liao2020survey}. In \cite{brafman2018ltlf,de2019foundations}, non-Markovian reward formulas expressed in \LTLf/\LDLf\ are translated into minimal DFA. The extended MDP is obtained as the synchronous product of the original MDP with these DFA components, guaranteeing minimality; because the DFA can be progressed symbol-by-symbol, the construction can also be carried out on-the-fly. The approach was refined in \cite{de2020temporal}, which merges the individual DFA into a single reward transducer, which can provide exponential savings. This compact transducer supports the four-valued runtime-monitoring semantics for \LTLf. Furthermore, our reward monitor can be constructed in linear time. Prior work has integrated $\textsc{gym}$ environments with reward monitors using the runtime monitoring language \cite{unniyankal2023rmlgym, hasanbeig2020deep}. 
Compared to these works, our approach has rewards that are densely provided with quantitative information, which is critical for addressing sparse rewards and sample efficiency.

\emph{Reward Machines} \cite{icarte2022reward,icarte2018using} represent non-Markovian rewards as Mealy machines whose edges are represented by modelled events. This makes them intuitive when crisp, symbolic events are available and supports automated potential-based shaping once the reward machine is given, 
although the effectiveness of the automated shaping  is limited to simple monitors 
as the states and actions of the environment are not considered in the shaping, merely desired paths through the monitor provide further rewards.
Recent works include extensions with counting automata \cite{bester2023counting}, first order representations \cite{ardon2024form}, and learning of reward machines \cite{parac2024learning} for inverse RL.

\paragraph*{Quantitative semantics.} \citet{li2017reinforcement} introduces quantitative semantics for robustness, but requires the whole trace to produce a robustness reward, therefore not being useful against sparse rewards and for long-horizon tasks.  \citet{balakrishnan2019structured} use quantitative Signal Temporal Logic (STL), 
for reward shaping, but assumes a bounded window length to keep complexity low and uses unnormalised atoms. \citet{hamilton2022training} outsources STL monitoring to an external tool making compositionality non-trivial, which is crucial for safety specifications and adequately handling non-Markovian properties, also again relying on unnormalised atoms. \citet{jothimurugan2019composable} use quantitative semantics, but the reward produced is not dense: a raw reward of $-\infty$ is given until the final monitor state, thus the method is based on reward shaping, and assumes bounds are supplied. There are scaling concerns, as each monitor state involves a separate policy head neural network.  

Numeric reward machines are introduced by \citet{levina2024numeric} which retain Boolean transitions but introduce negative distance-to-goal formulas for providing denser rewards. However, quantitative temporal semantics are not used, nor is any synthesis algorithm provided to construct such machines from a formal specification.

Finally, all of the aforementioned papers overlook safety in the compositionality process, whereby safety properties are able to veto all rewards globally.

\section{Conclusions and Future Work}
\label{sec:conclusion}
In this paper, we have shown how to construct reward monitors using quantitative, linear time \QLTLf specifications, and then demonstrated the effectiveness of our approach against Boolean monitors and handcrafted reward functions, across various environments with differing reward characteristics. We empirically showed that our quantitative monitor surpasses or matches the performance of a Boolean monitor, and in some cases, surpasses the performance of the manually-specified reward function. Learning uses the product between the environment MDP and the composite reward monitor formed from all specification-reward pairs. Additionally, we syntactically identify which monitors target safety properties, which when violated, have the power to block rewards received from all other monitors indefinitely.

For future work, it would be interesting to consider discounted rewards which could form part of the specification tuple, or be used more directly through discounted \LTL \cite{alur2023policy}. Quantitative temporal operators could also be incorporated, such as in \cite{frigeri2012fuzzy,mu12024checking}, where $\Eventually_t$ denotes eventually in the next $t$ instants ("within"), similarly for always within the next $t$ instants, as well as other more complex operators such as "almost always" and "soon", "nearly always", "gradually", which use fuzzy semantics, and may assist in learning convergence. Another aspect that could be investigated is robustness specifications \cite{anevlavis2022being}, which could help with other aspects such as curriculum learning (i.e., being confident that an agent has learnt a principal sub-task prior to advancing to a more difficult task). Our approach would also be suitable for the multi-agent setting, where \QATLf \cite{ferrando2024theory} could be used.


\section*{Acknowledgments} The research described in this paper was
partially supported by the EPSRC (grant number EP/X015823/1) and the UKRI Centre for Doctoral Training in Safe and Trusted AI (grant number EP/S0233356/1). We also thank Jialue Xu and Alexander Philipp Rader for contributing some of the original code for the reward monitors.


\bibliography{cite}

\clearpage

\ifwithappendix
\appendix
\section*{Appendix}

\section{Background}
\label{sec:appendix_background}

The following provides the semantics of how a \QLTLf formula is interpreted with a real-value between $[0, 1]$. The remaining (abbreviated operators) not found in the main matter can be defined as the following, given in simplified form:

$$
\begin{aligned}
    [\![\false, i]\!](\lambda) &= 0,\\
    [\![\varphi_{1}\vee \varphi_{2}, i]\!](\lambda) &= 
      \max\bigl([\![\varphi_{1}, i]\!](\lambda),\,[\![\varphi_{2}, i]\!](\lambda)\bigr),\\
    [\![\Eventually\,\varphi, i]\!](\lambda) &= \max_{\,i \le j \le |\lambda|} \bigl[\![\varphi, j]\!\bigr](\lambda),\\
    [\![\Always\,\varphi, i]\!](\lambda) &= \min_{\,i \le j \le |\lambda|} \bigl[\![\varphi, j]\!\bigr](\lambda).
\end{aligned}
$$

\section{Quantitative Reward Monitor}
\label{sec:appendix_quantitative}

\subsection{Construction}
\label{sec:appendix_quantitative_construction}
Connective cases that are abbreviations can be defined as follows:
\begin{itemize}
   \item For $\varphi = \false$, the monitor is assigned states $\StatesSet \gets \{q_0\}$, set of register-value pairs $\mathcal{V} \gets \bigl\{(t_{\false}, 0)\bigr\}$, transition function $\TransitionFn \gets \bigl\{(q_0, q_0)\bigr\}$, register update function $\delta_v \gets \emptyset $, and reward function $\delta_r \gets \mathcal{V}(t_\bot)$.

  \item For disjunction $\varphi = \varphi_1 \lor \varphi_2$, this is constructed the same as conjunction, tracking the $\max$ instead of $\min$.
\end{itemize}

\subsubsection{Optimisation}
Note that, for temporal eventually $\Eventually$ and always $\Always$, the respective abbreviations provided can be used. Alternatively, with effort to reducing the number of registers, one could track the $\max$ and $\min$ of the formulas respectively as a performance optimisation. Likewise, when either $\Eventually \varphi$ or $\Always \varphi$ are nested inside the other, only the latest value of $\varphi$ needs to be tracked.

Since quantitative monitors have a deterministic transition function, we can let $q_{-1}$ and $q_{-2}$ denote the final state of a monitor and the penultimate state of a monitor respectively. As with temporal until $\mathcal{U}$, all contained atomic registers are \emph{initialised} using the initial state labels $\mathcal{L}_p(\mathcal{S}_0)$ and are \emph{updated} on each step using the successor-state labels $\mathcal{L}_p(s')$, and registers are renamed to avoid name clashing.

\begin{itemize}
  \item For temporal always $\varphi = \Always \psi$
  \begin{itemize}
    \item If $\psi = \Eventually \varphi$, then $\mathcal{A}^{\varphi} \gets \textsc{Synth}(\varphi)$, $\StatesSet \gets \StatesSet^{\varphi}$, $\InitState \gets q_0^{\varphi}$, continuous updates are supported by using the renamed registers and adding a loop between the penultimate and ultimate state: $\TransitionFn \gets \delta_q^{\varphi}$, $\TransitionFn(q_{-1}) \gets q_{-2}$, $\RegisterFn(q_{-2}) \gets \RegisterFn(q_{-1})$, $\RegisterSet\gets \RegisterSet^{\varphi} \cup \{ (t_{\Always (\Eventually \varphi)}, 0) \}$, and $\RegisterFn \gets \delta_v^{\varphi}$. Then for all $q \in \StatesSet$, do $\delta_v(q, t_{\Always(\Eventually \varphi)}) \gets \RegisterSet(t_{\varphi})$. And finally, $\delta_r \gets \RegisterSet(t_{\Always(\Eventually \varphi)})$.
    \item Otherwise, $\mathcal{A}^{\psi} \gets \textsc{synth}(\psi)$, $\StatesSet \gets \StatesSet^{\psi}$, $q_0 \gets q_0^{\psi}$, $\delta_q \gets \delta_q^{\psi}$, continuous updates are supported by using the renamed registers and adding a loop between the penultimate and ultimate state: $\TransitionFn(q_{-1}) \gets q_{-2}$, $\RegisterFn(q_{-2}) \gets \RegisterFn(q_{-1})$, $\RegisterSet \gets \RegisterSet^{\psi} \cup \{(t_{\Always \psi}, 1)\}$, $\delta_v \gets \delta_v^{\psi}$. Then, for all $q \in \StatesSet$ do: $\delta_v(q, t_{\Always \psi}) \gets \min(\RegisterSet(t_{\Always_\psi}), \RegisterSet(t_{\psi}))$. Finally, $\delta_r \gets \RegisterSet(t_{\Always \psi})$.
  \end{itemize}
  \item For temporal eventually $\varphi = \Eventually \psi$
  \begin{itemize}
    \item For the case of $\psi = \Always \varphi$, then: $\mathcal{A}^{\varphi} \gets \textsc{synth}(\varphi)$, $\StatesSet \gets \StatesSet^{\varphi}$, $q_0 \gets q_0^{\varphi}$, $\TransitionFn \gets \delta_q^{\varphi}$, continuous updates are supported by using the renamed registers and adding a loop between the penultimate and ultimate state: $\TransitionFn(q_{-1}) \gets q_{-2}$, $\RegisterFn(q_{-2}) \gets \RegisterFn(q_{-1})$, $\RegisterSet \gets \RegisterSet^{\varphi} \cup \{(t_{\Eventually(\Always \varphi)}, 0)\}$, $\delta_v \gets \delta_{v}^{\varphi}$. Then for all $q \in \StatesSet$, do $\delta_v(q, t_{\Eventually(\Always \varphi)}) \gets \RegisterSet(t_{\varphi})$. Finally, $\delta_r \gets \RegisterSet(t_{\Eventually (\Always \psi)})$.
    \item Otherwise, $\mathcal{A}^{\psi} \gets \textsc{synth}(\psi)$, $\StatesSet \gets \StatesSet^{\psi}$, $q_0 \gets q_0^{\psi}$, $\TransitionFn \gets \delta_q^{\psi}$,  continuous updates are supported by using the renamed registers and adding a loop between the penultimate and ultimate state: $\TransitionFn(q_{-1}) \gets q_{-2}$, $\delta_v(q_{-2}) \gets \delta_v(q_{-1})$, $\RegisterSet \gets \RegisterSet^{\psi} \cup \{(t_{\Eventually \psi}, 0)\}$, $\delta_v \gets \delta_v^{\psi}$. Then, for all $q \in \StatesSet$ do: $\delta_v(q, t_{\Eventually \psi}) \gets \max(\RegisterSet(t_{\Eventually \psi}), \RegisterSet(t_{\psi}))$. Finally, $\delta_r \gets \RegisterSet(t_{\Eventually \psi})$.
  \end{itemize}
\end{itemize}

In the case where always $\Always$ and eventually $\Eventually$ are nested into each other, such that $\Eventually(\Always(\varphi))$ or $\Always(\Eventually(\varphi))$, this is simplified into $last(\varphi)$ by the \QLTLf semantics. As we are using runtime monitoring, the current trace index, $s'$ from a transition tuple ($s, a, s'$) corresponds to the last trace index.

\subsection{Composite Reward Monitor} \label{sec:appendix_composite_qrm}
A single merged reward monitor is provided by Definition~\ref{def:composite_reward_monitor}.
\begin{definition}(Composition of quantitative monitors)  
\label{def:composite_reward_monitor}
Let $\mathcal{A}_c$ denote the composition of reward monitors
$\{\mathcal{A}_i\}_{i=0}^{n}$, where $n$ is the number of monitors.
\begin{itemize}
    \item State-space, $\StatesSet_c \eqdef \StatesSet_0 \times \StatesSet_1 \times \dots \times \StatesSet_n$
    
    \item Tuple of initial states from each monitor, such that  $(q_0^0, \: q_0^1, \: \dots, \: q_0^n)$
    
    \item Register-value pairs, $\mathcal{V} \eqdef \mathcal{V}_0 \cup \mathcal{V}_1 \cup \dots \cup \mathcal{V}_n$.
    
    \item State-transition function, $   \TransitionFn\left(q_0,\dots,q_n\right) \eqdef$
    $ \\
    \left(\delta_{q_0}(q_0), \dots, \delta_{q_n}(q_n)\right)
    $
    
    \item Register-update function, $ \delta_v\left(\left(q_0, \ldots, q_n\right), t\right)\eqdef$ 
    $ \\
    \begin{cases}\delta_{v_0}\left(q_0, t\right) & \text { if } t \in \mathcal{V}_0 \\ \ldots & \\ \delta_{v_n}\left(q_n, t\right) & \text { if } t \in \mathcal{V}_n\end{cases}$
    
    \item Reward function, $\delta_r : \mathcal V \to \mathbb R \eqdef$
    $ \\
    \begin{cases}
      \zeta \in \mathbb{R}, & \hspace{-1.25cm} \parbox[t]{.6\textwidth}{\ \ \ \ \text{if any safety property is violated ($\exists s. \: \varphi_{s} = \bot $)}},\\[6pt]
      \displaystyle \sum_{i=0}^{n} \delta_{r_i}(\mathcal V_i)
        & \text{otherwise.}
    \end{cases}
    $
\end{itemize}
\end{definition}

\subsection{Proof of Lemma 1}
\label{sec:lemma_proof}

\begin{proof}(By structural induction) In the base cases, $\true$, $\false$, and every atomic proposition $p$, the procedure \textsc{synth} explicitly constructs and returns all components of the corresponding QRM. For the inductive step, assume \textsc{synth} yields correct QRMs for the immediate subformulae; then, for each composite \QLTLf operator, it combines these monitors according to the operator's semantics, thereby producing a valid QRM for the entire formula.
    \label{proof:lemma_synth_qltlf}
\end{proof}

\subsection{Proof of Correctness}
\label{sec:correctness_proof}

\begin{proof}
  We prove this by structural induction, as we can enumerate the cases involved in a \QLTLf formula and its monitor construction. Recall that the monitor is invoked after an action is performed, involving the triple $(s, a, s')$. Also recall that, since a quantitative reward monitor provides runtime monitoring, the finite trace $\lambda$ corresponds to the current execution trace.

  \textbf{Base cases.} The following outlines the base cases, which are atomic and trivial formulas:
  \begin{enumerate}
    \item $\varphi = \true$, we construct a monitor with a single state with a reward register for $\true$ always set to 1. The semantic value of the formula for every position $i$ is $\bigl[\![\true,i]\!\bigr](\lambda) = 1$.
    \item $\varphi = \false$, we construct a monitor with a single state with a reward register for $\false$ always set to 0. The semantic value of the formula for every position $i$ is $\bigl[\![\false,i]\!\bigr](\lambda) = 0$.
    \item $\varphi = p$, for an atomic proposition, the semantic value at a given time step $i$ is $\bigl[\![\,p,i\,]\!\bigr](\lambda)$. The monitor constructed transitions from the initial state to an empty self-looping state with no further register updates. The initial state provides a reward output of the current valuation of $p$ from the labelling function $\mathcal{L}_p(s)$ of the MDP, and this is stored permanently in the reward register, which corresponds to evaluating a proposition $p$ at index $i$.
  \end{enumerate}

  \textbf{Inductive cases.} 
  \begin{enumerate}
    \item Negation $\varphi = \lnot \psi$, by constructing the monitor $\mathcal{A}^\psi$, we simply append a register of $\lnot \psi$, updating in all states, which runs $1 - \mathcal{V}(\psi)$ (which is equivalent to the formula's interpretation), and use this as the reward register.
    \item Conjunction $\varphi = \psi_1 \land \psi_2$, by constructing the respective submonitors of $\mathcal{A}^\psi 1$ and $\mathcal{A}^\psi 2$, taking the Cartesian product of the state and transition space, tracking the propositional values of each formula, we append an additional register that tracks the minimum between the two formulae such that $\min(\psi_1, \psi_2)$ and use this as the reward register, which provides the correct output for time step $i$.
    \item Disjunction is the same as conjunction, tracking the $\min$ instead of $\max$.
    \item Next $\varphi = \Next \psi$ is almost identical to the base case of an atomic proposition, but requires accessing the next timestep $i+1$ instead of $i$. Whilst requiring a future time step's value is typically an issue with runtime verification, when performing an action we have the triple ($s, a, s'$), therefore the labelling of $s'$ is considered for $i+1$.
    \item Until $\varphi = \psi_1 \: \Until \: \psi_2$, is determined by the balance of (i) how strongly $\varphi_1$ holds in all steps until $\varphi_2$ holds, and (ii) how strongly $\varphi_2$ holds at that point. This is equivalent to taking the maximum value of $\min(\min \varphi_1, \max \varphi_2)$ throughout the trace, which is how the corresponding reward register is set up and what is being tracked.
    \item Release $\varphi = \psi_1 \Release \psi_2$ is handled similarly to Until, forming the required registers to take the minimum value of $\max(\max \varphi_1, \min \varphi_2)$ throughout the trace, which is how the corresponding reward register is set up.
  \end{enumerate}
\end{proof}

\subsection{Proof of Linearity}
\label{proof_linearity}

\begin{proof}
    It is clear that the base cases of construction, that is, an atom $p, \top, \bot$ only add a maximum of 2 states and 2 transitions.
    
    For any inductive case, namely a binary or temporal connective that is processed, the algorithm forms the product of the operand monitors, but every operand monitor it produces is a lasso (automata-theoretic concept of a finite chain whose tail is an infinite loop, in this case a 1-state or 2-state loop). Because such chains advance in lock-step, each input symbol can add at most one previously unseen composite state before all operands sit in their final loop. Hence the total number of states in the resulting monitor grows linearly in the formula size.
\end{proof}

\subsection{Markovian Policy Proof}
\label{sec:markovian_policy_proof}

In this section we prove Theorem~\ref{theorem:markovian_policy}, namely that an optimal Markovian policy on the synchronous product MDP $\ExtendedMDP$ is sufficient to optimally realise the non-Markovian objectives encoded by the quantitative reward monitor. We first introduce Non-Markovian Reward Decision Processes in Definition~\ref{def:nmrdp}.
\begin{definition}[NMRDP] \label{def:nmrdp}
  A \emph{Non-Markovian Reward Decision Process} is defined as 
  $\mathcal{N} = (\mathcal{S}, \mathcal{A}, \Pr, \gamma, \mathcal{R})$, where $\mathcal{R}$ can depend on the entire history of state-action pairs, such that $\mathcal{R} : (\mathcal{S} \times \mathcal{A})^* \rightarrow \mathbb{R}$. The other elements are inherited unchanged from an MDP.
\end{definition}

In our setting, $\mathcal{R}$ is given compositionally by a set of specification-reward pairs $\{(\varphi_i,\rho_i)\}_{i=1}^n$: each $\rho_i$ contributes based on the quantitative satisfaction of the formula $\varphi_i$ on the trace.

The proof of Theorem~\ref{theorem:markovian_policy} is a straightforward adaptation of results in \cite{de2020temporal,brafman2018ltlf}, but we restate the results for completeness.

We follow the standard notion of equivalence between an NMRDP and an extended MDP as in \citet{brafman2018ltlf,de2020temporal}.

\begin{definition}\label{def:equivalence_nmrdp}
  As in \cite{brafman2018ltlf}, a NMRDP $\mathcal{N} = (\mathcal{S}, \mathcal{A}, \Pr, \gamma, \mathcal{R})$ is equivalent to an extended MDP $M' = (\mathcal{S}', \mathcal{A}, \Pr', \mathcal{R}', \gamma)$, if two functions exist $t: \mathcal{S}' \to \mathcal{S}$ and $\sigma: \mathcal{S} \to \mathcal{S}'$, such that:
  \begin{enumerate}
    \item $\forall s \in \mathcal{S}: t(\sigma(s)) = s$.
    \item $\forall s_1, s_2 \in \mathcal{S}$, $s_1' \in \mathcal{S}'$, and $a \in \mathcal{A}$:
          if $\Pr(s_1, a, s_2) > 0$ and $t(s_1') = s_1$, there exists a unique $s_2' \in \mathcal{S}'$ such that $t(s_2') = s_2$ and $\Pr'(s_1', a, s_2') = \Pr(s_1, a, s_2)$.
    \item For any feasible trajectory  $(s_0, a_1, \ldots, s_{n-1}, a_n)$ of $\mathcal{N}$ and $(s_0', a_1, \ldots, s_{n-1}', a_n)$ of $M'$, such that $t(s_i') = s_i$ and $\sigma(s_0) = s_0'$, we have $\mathcal{R}(s_0, a_1, \ldots, s_{n-1}, a_n) = \mathcal{R}'(s_0', a_1, \ldots, s_{n-1}', a_n)$.
  \end{enumerate}
\end{definition}

Informally, in Definition~\ref{def:equivalence_nmrdp}, $t$ projects extended states to original states, $\sigma$ embeds original states into the extended model, and the three conditions express that transitions and rewards are preserved along corresponding trajectories.

\begin{lemma}[Policy equivalence] \label{lemma:markovian_equivalence}
  \cite{de2020temporal}
  Let $\mathcal{N}$ and $M'$ be equivalent in the sense of Definition~\ref{def:equivalence_nmrdp}.
  A policy $\pi$ for $\mathcal{N}$ and a policy $\pi'$ for $M'$ are \emph{equivalent} if, for every feasible trajectory, they induce the same sequence of rewards. Equivalent policies have the same value for every initial state.
\end{lemma}

\begin{proof}  
Consider the NMRDP $\mathcal{N} = (\mathcal{S}, \mathcal{A}, \Pr, \gamma, \mathcal{R})$ where $\mathcal{R} = \{(\varphi_i,\rho_i)\}_{i=1}^n$. Let $\mathcal{A}$ be the composite quantitative reward monitor obtained from these specification-reward pairs, and let $\ExtendedMDP = \mathcal{M} \otimes \mathcal{A}$ be the synchronous product MDP of Definition~\ref{def:composite_reward_monitor}. Recall that the state space of $\ExtendedMDP$ is $\mathcal{S}' = Q \times \mathcal{S}$, where $Q$ is the monitor state space, and that its transition kernel and reward function are Markovian on $\mathcal{S}'$.

We show that $\mathcal{N}$ and $\ExtendedMDP$ are equivalent in the sense of Definition~\ref{def:equivalence_nmrdp}. Define $\sigma(s) =  (q_0, s)$ and $t(q,s) = s$. The three conditions hold by construction of the product: the $\mathcal{S}$ component evolves according to $\Pr$, while the monitor state $q$ is updated deterministically from the labelled transition, and the reward in $\ExtendedMDP$ coincides with the monitor output, hence with $\mathcal{R}$ along corresponding trajectories.

Now let $\pi$ be any policy on $\ExtendedMDP$, and define a policy $\pi'$ on $\mathcal{N}$ that, on a history ending in $s_t$, selects the same action as $\pi$ in the corresponding extended state $(q_t,s_t)$ obtained by progressing the monitor along that history. By Lemma~\ref{lemma:markovian_equivalence}, $\pi$ and $\pi'$ are equivalent policies and thus have the same value for every initial state. In particular, if $\pi$ is optimal on $\ExtendedMDP$ (with respect to the non-Markovian objectives encoded by the monitor), then $\pi'$ is optimal for the original NMRDP. Hence, any optimal policy for the non-Markovian objectives can be realised by a Markovian policy on the extended MDP $\ExtendedMDP$, which proves Theorem~\ref{theorem:markovian_policy}.
\end{proof}

\section{Environment Specification Details}
\label{sec:appendix_env_spec_details}

This section outlines the domain-specific quantitative fluents defined, per environment evaluated in our paper. Per Table~\ref{table:cw}, \textsc{Classic} and \textsc{Box} environments exhibit greater performance compared to \textsc{Gridworlds} and \textsc{Toy}, since the latter two comprise small discrete grids where quantitative pathfinding signals often favour regions that are ultimately unproductive (for instance, high reward-density areas cluttered with obstacles). The effectiveness of quantitative reward monitoring relies on being able to define strong quantitative progress measures for each environment.

In all experiments, we used an exponentially moving average convergence test (see Definition~\ref{def:reward_convergence}) with span $N = 32$ episodes, forming checkpoints every $N$ episodes and declaring convergence once the last $P = 5$ checkpoint intervals showed less than a 1 relative improvement in the EMA of returns. The relative tolerance was automatically adapted to the recent noise in the checkpoint sequence, clipped to the range $[\rho_{\min}, \rho_{\max}] = [0.002, 0.02]$, and we did not enforce a minimum EMA level or exclude negative returns.

\subsection{Classic Environments}
\paragraph*{Acrobot} models a two-link under-actuated pendulum—only the elbow joint is powered—tasking an agent with swinging the lower link up to a target height, making it a benchmark for control under limited actuation. There are six observations: $[\cos(\theta_1), \sin(\theta_1), \cos(\theta_2), \sin(\theta_2), 
\dot{\theta_1},
\dot{\theta_2}
]$. 

This environment lacks any coordinates of the rotational joints directly, only specifying angles of two links with uniform length of 1. However, it is still possible to calculate the x and y values of the tips of links due to having all necessary angles of the links.
For calculating the tip height (y), we can use the cosine addition formula: $\cos(\theta_1 + \theta_2) = \cos\theta_1 \cos\theta_2 - \sin\theta_1 \sin\theta_2$, and then define tip height as: $h_{\text{tip}} = -\cos\theta_1 - \cos(\theta_1 + \theta_2)$. To normalise between $[0,1]$, we can do $reach\_goal = y_{\text{clamped}} = \min\{\max\{ h_{\text{tip}}, 0 \}, goal \}$, where the goal height is typically 1 as that terminates the episode successfully.

We can also provide a reward for reaching a high velocity (which is required for swinging upwards), such that:

$$
velocity =
\begin{cases}
0, & \text{if } \dot{y}_{\text{tip}} \leq 0, \\[6pt]
\min\!\left(\dfrac{\dot{y}_{\text{tip}}}{v_{\text{max}}},\,1\right), & \text{if } \dot{y}_{\text{tip}} > 0.
\end{cases}
$$
where $v_{max}$ can be assumed to be 2.0.

For the Boolean monitor, we can construct from the following specification: $[(\Eventually(reach\_goal), 10), (\Always \true, -1)]
$, the former providing a reward for reaching the goal tip height and the latter providing a standard time penalty to entice a faster solution.

For the quantitative monitor, we use the following specification: $[\bigl(\Eventually(reach\_goal), 50\bigr), \bigl( \Eventually(velocity), 10\bigr), \\ \bigl(\eveAlwTrue, -1 \bigr)]$.

\paragraph*{Cartpole} requires balancing a pole in a cart, where this pole is attached by an un-actuated joint moving on a frictionless track, where the cart can be moved in either the left or right directions. By default, the reward function provided by \textsc{Gymnasium} does not consider reaching an optimal position. Rather, the base reward function only provides a reward of 1 at each time step, which optimises towards remaining balanced since termination occurs when unbalanced. We enhance the base reward function to also optimise towards reaching a goal position.

The positional observation of \textsc{Cartpole} is only defined in terms of the x-axis. Therefore, we can define the distance between the cart ($x$) and the goal position ($g$) as $d = |x-g|$. We can define the proximity term $p_t$ as follows:

$$
p_t \;=\;
\begin{cases}
1, & d \le \tau,\\[6pt]
1 - \dfrac{d - \tau}{X - \tau}, & \tau < d < X,\\[10pt]
0, & d \ge X.
\end{cases}
$$

$\tau$ is a tolerance and $X$ is a threshold on the track's boundaries (defined as 2.4). Therefore, the reward per timestep is $r = 1 + p_t$.

The relevant reward-specification pairs for \textsc{Cartpole} can be defined as the following for both the Boolean and quantitative settings: $[(\Eventually(\Always(reach\_goal)), 2), \: (\Always(balanced), 4)]$, with the former expressing a persistence property and the latter a safety property. For the first specification, we want to ensure that we always remain in a goal position (i.e., do not slip backwards, here we can assume $goal\_pos = 2$). The second specification requires always remaining balanced, which is useful regardless of how far the agent is from the goal position. 

  The quantitative atoms can be defined as:
  \begin{eqnarray*}
      \textit{balanced} & = &
      \begin{cases} 
      \frac{0.209 - |angle|}{0.209}, & \text{if } |angle| \leq 0.209 \\ 
      0, & \text{else}
  \end{cases}\\
  %
      \textit{reach\_goal} & = &
      \begin{cases} 
      \frac{current\_pos}{goal\_pos}, & \text{if } current\_pos > 0.01 \\ 
      0, & \text{else}
      \end{cases}
  \end{eqnarray*}
  
In the case of Boolean atoms, $balanced$ returns $\true$ iff $angle \leq 0.209$, else $\false$. As for $reach\_goal$, it is true if the agent is within $\leq 0.1$ of the $goal\_pos$. The task completion measure (performance function) is specified as an equally weighted mean of $reach\_goal$ and $balanced$.

\paragraph*{Mountain car} challenges an agent to reach the goal flag atop a steep hill in the shortest time possible. Because the engine is under-powered, the agent must first move away from the goal to the opposite slope, using the descent to build momentum before driving up the final ascent.

The observation space is defined with $\texttt{Box}(-1.2, 0.6)$ which corresponds to the $x$ coordinate of the car's current position. 
Notably, the $goal\_pos$ would typically be set at 0.5 as this causes successful episode termination. The observation space covers a grid using a negative coordinate system. Therefore, the $reach\_goal$ fluent can be defined in the following manner:
$$
\begin{cases}
0, & \text{if } x \le x_{\min}, \\[8pt]
\min\!\Biggl\{ \max\!\Biggl\{ \dfrac{x - x_{\min}}{x_{\text{milestone}} - x_{\min}},\, 0 \Biggr\},\, 1 \Biggr\}, & \text{if } x > x_{\min}.
\end{cases}
$$ 
where $x_{milestone}$ is the goal position that can be set, and $x_{min}$ is the smallest reasonable $x$ value before outputting $0$, however, using $-1.2$ (the environment's minimum value) is useful enough.

The second parameter in the observation space is the velocity, which is useful as a heuristic. Before reaching the goal, a small velocity is not helpful, and when near the goal, coming to a rest is ideal at any arbitrarily set goal position (as opposed to overshooting). The max velocity $v_{max} = 0.07$ allows defining the $velocity$ fluent:

$$
\begin{cases}
\max\Bigl\{ 0,\; 1 - \dfrac{|v|}{v_{\max}} \Bigr\}, & \text{if } x \ge g, \\[1em]
0, & \text{if } x < g \text{ and } v \le 0, \\[1em]
\min\Bigl\{ \dfrac{v}{v_{\max}},\; 1 \Bigr\}, & \text{if } x < g \text{ and } v > 0.
\end{cases}
$$

Therefore, the quantitative specification is: $[\bigl( \Eventually(reach\_goal), 50 \bigr), \bigl( \Eventually(velocity), 25 \bigr)]$, and the Boolean specification is simply $[\bigl( \Eventually(reach\_goal), 50 \bigr)]$, as transforming the $velocity$ fluent into a crisp Boolean value is non-trivial. For the task completion evaluation metric, the $reach\_goal$ fluent suffices, since $velocity$ is an intrinsic reward acting as a heuristic.

\paragraph*{Pendulum}  involves a problem that originates from control theory, with the idea to swing a pendulum upright, maintaining its centre of gravity above the fixed point, from any starting angle and angular velocity.  The reward function specified by \gym is already informative, $r= -(\theta^2 + 0.1 * \theta_{dt^2} + 0.001 * torque^2)$, where $\theta$ is the pendulum's angle normalised between $[-\pi, \pi]$. We opt to use the following quantitative and Boolean specification:
$[\bigl(\Eventually(upright), 25\bigr), \bigl(\Eventually(\Always(upright \land stabilised)), 10 \bigr), \bigl(\eveAlwTrue, -1\bigr)]$.

For the Boolean case, from observations we can extract $\cos(\theta_t),\, \sin(\theta_t),\, \text{and} \: \dot{\theta}_t$, where $\dot{\theta}_t$ represents the torque. We can state $stabilised \iff \dot{\theta}_t =1$, and $upright  \iff \cos(\theta_t) = 1$. Quantitatively, being upright can be defined as: $cos\_upright\_score(\cos_t, \sin_t, \dot{\theta}_t)
\;=\;
\frac{\cos_t + 1}{2}$. And stabilised can be defined as: $stabilised(\cos_t, \sin_t, \dot{\theta}_t)
=
\max\bigl(
  \min(\,1 \;-\; \tfrac{|\dot{\theta}_t|}{\text{max\_speed}},\; 1),\;0
\bigr)$. The task evaluation metric allocates an equal weighting of uprightness and stability, i.e. $0.5 \cdot cos\_upright\_score(\cos_t, \sin_t, \dot{\theta}_t) + 0.5 \cdot stabilised(\cos_t, \sin_t, \dot{\theta}_t)$.

\subsection{Toy Environments}
\paragraph*{Cliff walking} requires reaching a goal position without falling into a cliff tile (which forces episode termination), with holes distributed around the bottom of the map and the goal position being in the bottom right.

For both the Boolean and quantitative monitors, we consider the following set of reward-specification pairs $[\bigl(\Eventually(\Always(reach\_goal)), 25 \bigr), \bigl(\Eventually(\Always(reach\_cliff)), -25 \bigr), \\ \bigl(\eveAlwTrue \land \lnot reach\_goal, -1\bigr)]$, which penalizes falling from the cliff and encourages reaching the goal position. We also add a time penalty to encourage minimal time by distributing -1 for each time step.

The observation is encoded as integer where $obs = current\_row * ncols + current\_col$ where both row and col are 0-indexed. The x position can be extracted by $obs \bmod 12$, and the y position can be extracted by the floor division of $\lfloor \frac{\text{obs}}{12} \rfloor$, as there are 12 numbers of columns. 

A quantitative reward monitor is able to exploit the Manhattan distance or \emph{breadth-first search} (BFS) distance, the latter calculating the distance to a goal position with obstacles that need to be traversed accounted for. However, in practice, this does not lead to the quantitative monitor outperforming the Boolean monitor, likely for the following reasons:
\begin{enumerate}
  \item The hole tiles are surrounded by a high density of rewards, encouraging the agent to explore a non-fruitful area.
  \item Combined with the above factor, the grid is already very small, so finding a path to follow is not difficult.
\end{enumerate}

Therefore, both the quantitative monitor and Boolean monitor use Boolean atoms, and should have near-identical performance. The task completion evaluation function is based on the $reach\_goal$, which is the clipped BFS distance scaled in $[0, 1]$, or $0$ if in a failure state (reached a hole tile).

\paragraph*{Frozen lake} is similar to cliff walking, however, the obstacles (lake tiles) are distributed more sparsely throughout the grid. We opted to use \texttt{is\_slippery = False} as the reward behaviour is not changed by having a slippery environment. Similar to the Cliff Walking environment, for both the Boolean and quantitative we consider the set of reward-specification pairs as $[\bigl(\Eventually(reach\_goal), 10 \bigr), \bigl(\Always(\lnot reach\_hole), -10 \bigr), \\ \bigl(\eveAlwTrue, -1\bigr)]$, which penalizes falling into the lake and encourages reaching the goal position in the shortest possible time. We also use the Boolean atoms for the quantitative monitor in this environment.

Since we operate using a 4x4 grid, we again use the BFS distance due to the presence of obstacles (instead of e.g. the negative Manhattan distance), where the maximum BFS distance of the grid is 6 (which is used to scale the BFS distance between $[0, 1]$). The distance is used as the task completion metric.

\paragraph*{Taxi} requires picking up a passenger from a given location and dropping them off at the goal position, avoiding obstacles in the route.

For the Boolean monitor, we use the following specification: $\bigl[
  \bigl(\Eventually(\textit{reach\_goal})\,,\; 100 \bigr),\;
  \bigl(\Eventually(\textit{at\_passenger})\,,\; 30 \bigr),\; \\
  \bigl(\Always\!\bigl(\Eventually(\textit{hit\_wall})\bigr)\,,\; -50 \bigr),\; \\
  \bigl(\Always(\textit{act\_drop\_off} \land \lnot \textit{has\_passenger})\,,\; -50 \bigr),\; \\
  \bigl(\Always(\textit{act\_drop\_off} \land \lnot \textit{at\_destination})\,,\; -25 \bigr),\; \\
  \bigl(\Always(\textit{act\_pick\_up} \land \lnot \textit{at\_passenger})\,,\; -25 \bigr),\;
  \bigl(\eveAlwTrue,\; -1 \bigr)\bigr]
$
Whereas for the quantitative monitor, we use the following specification which involves slight variations of the above,  $\bigl[
  \bigl(\Eventually\Always(\textit{reach\_goal})\,,\; 100 \bigr),\;
  \bigl(\Eventually(\textit{at\_passenger})\,,\; 30 \bigr),\; \\
  \bigl(\Always\!\bigl(\Eventually(\textit{hit\_wall})\bigr)\,,\; -50 \bigr),\; \\
  \bigl(\Eventually\Always(\textit{act\_drop\_off} \land \lnot \textit{has\_passenger})\,,\; -50 \bigr),\; \\
  \bigl(\Eventually\Always(\textit{act\_drop\_off} \land \lnot \textit{at\_destination})\,,\; -25 \bigr),\; \\
  \bigl(\Eventually\Always(\textit{act\_pick\_up} \land \lnot \textit{at\_passenger})\,,\; -25 \bigr),\; \\
  \bigl(\eveAlwTrue,\; -1 \bigr)\bigr]
$

For the quantitative monitor, we again opt to use Boolean atoms as with the other \textsc{Toy} environments. Also note that, the atom $hit\_wall$ and action atoms $act\_...$ can be viewed as a limited form of non-Markovian information encoded into the state $s'$ of a transition $(s, a, s')$. Specifically, $hit\_wall$ checks if $s' = s$ and no other action was performed (like a drop-off or pick-up).

In the quantitative setting, we are able to exploit the distance to the passenger ($reach\_passenger$) and the distance to the drop-off point ($reach\_destination)$ in a quantitative fashion. The BFS distance can be used here to account for obstacles in the route, but note that walls are not given explicit coordinate indexes. Rather checks for wall collisions involves specific pairs of coordinates being traversed.

For the task evaluation metric, it is calculated in two halves: 0.5 is assigned once the passenger is picked up; or the BFS distance to the passenger otherwise. Then, following the passenger being picked up, the final 0.5 (for a total of up to 1.0) is assigned once the passenger is dropped-off; or the BFS distance to the drop-off location otherwise.

\subsection{Box2D Environments} 

\paragraph*{Bipedal walker} is an environment where a 4-jointed robot must be controlled to reach the end position of the map, avoiding falling over.  There are 24 observations, however, the robot's position coordinates are not exposed and not easily calculable.

There are two main fluents that we look at. The first is $upright$, extracted from the hull angle. Normalised, it takes the following form $upright = \max\!\Bigl\{ \min\!\Bigl\{ 1 - \frac{|\theta|}{\alpha},\, 1 \Bigr\},\, 0 \Bigr\}$, where $\alpha$ is the maximum acceptable angle for being upright, defined as $0.4$ in this case.

The second atom is $smooth$, defined as $S(v_h) = \max\Bigl\{\min\Bigl\{ s, \,1 \Bigr\},\, 0\Bigr\} = \max\Bigl\{\min\Bigl\{ 1 - \frac{\left| v_h - 0.12 \right|}{0.1},\, 1 \Bigr\},\, 0\Bigr\}$, which checks for horizontal velocity close to 0.12, which translates to approximately 1.0 m/s.

The Boolean fluents use the quantitative calculations and return true only if $value >= 0.99$. The specification-reward pair for both the Boolean and quantitative setting can be simply defined as: $[\bigl(\Always(upright), 10 \bigr), \bigl( \Always(smooth), 10 \bigr)]$.

There are other interesting fluents that could be defined in this non-Markovian setting, such as checking for alternation of legs making contact with the ground.

The task completion evaluation metric is defined as $smoothness \ * 0.5 + upright * 0.5$.


\paragraph*{Lunar lander} aims to land a rocket ship, where a perfect landing requires that both legs make contact with the ground, in an upright position, within a defined area of the map between two legs. 

A quantitative reward can be defined based on Euclidean distance to the landing pad, and making a landing that is successful, whereby both legs make contact with the ground and the landing is 'soft'.

Landing closeness is based on Euclidean distance to the landing pad's center, which is at Cartesian coordinates (0, 0). Let $d_{\max}$ be the maximum distance where the score is 0 (for normalization purposes), then let $d$ be $d = \sqrt{(x - x_c)^2 + (y - y_c)^2}$, allowing $landing\_closeness = \max\Bigl\{ 0, \, 1 - \frac{d}{d_{\max}} \Bigr\}$. For the $legs\_contact$ fluent, this is provided as part of the observation space, in the quantitative setting, one leg on ground results in 0.5, both as 1.0, and neither leg scores 0. In the Boolean setting, both legs must be on ground to be true and false otherwise. $outside\_viewport$ is defined as y being greater than 1.

The fluent $upright$ is assigned as $\max\Bigl\{\min\Bigl\{ 1 - \frac{|\theta|}{\alpha}, \,1 \Bigr\},\, 0\Bigr\}$, where $\alpha$ is the $max\_angle$, defined for normalization purposes. Let $\omega$ denote the angular velocity (extracted from the observation vector). Given a maximum angular velocity $\omega_{\max} = 2.0$, the angular velocity score is defined as
$s = 1 - \frac{|\omega|}{\omega_{\max}}$. The final angular velocity score, clamped to the interval $[0,1]$, is then given by $ S(\omega) = \max\Bigl\{ \min\Bigl\{ 1 - \frac{|\omega|}{\omega_{\max}},\, 1 \Bigr\},\, 0 \Bigr\}$. The angular velocity score rewards a low angular velocity specifically for landing score measurement.

Now we can define $soft\_landing$, which only provides an active value near landing time to measure how smooth the landing is. 
$$
L = \text{landing\_closeness} \quad
U = \text{upright}.
$$
$$
A = \text{angular\_velocity\_score},  \quad
C = \text{legs\_contact\_score}. 
$$

Smoothness $S$ is defined as:
$$S \eqdef
\begin{cases}
0.33\, L + 0.33\, U + 0.33\, A, & \text{if } L(\mathbf{obs}) \geq 0.8, \\[1em]
0, & \text{otherwise.}
\end{cases}
$$

The progress measure can be defined as:
$M = 0.4 * L + 0.2 * C + 0.2 + U * 0.2 + A * 0.2$.

Thus, we can define the Boolean specification as: $[\bigl( \Eventually(\Always(reach\_landing\_zone)), 50 \bigr), \bigl(\Eventually(legs\_contact), 20\bigr), \\ \bigl(G(outside\_viewport), -20\bigr), \bigl(\Always(\true), -1\bigr)]$. And the quantitative specification as: $[\bigl( \Eventually(\Always(reach\_landing\_zone)), 50 \bigr), \\ \bigl(\Eventually(soft\_landing), 50 \bigr), \\ \bigl( \Eventually(\Always(outside\_viewport)), -20\bigr), \bigl(\eveAlwTrue, -1\bigr)]$.

\subsection{Safety Gridworlds Environments}

\paragraph*{Island navigation} is a simple 8x6 gridworld environment where the aim is to reach a goal tile while avoiding any water tiles (which force episode termination). The intent of the environment is \emph{safe exploration}, meaning that any unsafe behaviour is intended to be avoided even during the learning process. The agent already has as an observation the minimum Manhattan distance to any water tile from its current position. In this particular environment, moving away from the goal coincides with reaching the water, which simplifies the specifications required.

The fluent $in\_water$ is true when the agent is in a water tile, likewise $at\_goal$ is true when the agent reaches the goal tile. In this environment, the quantitative monitor also shares the Boolean atoms.

The Boolean reward-specification pairs are $[ \bigl(\Always(\lnot in\_water), 100\bigr), \bigl(\Eventually(at\_goal), 50 \bigr), \bigl(\Always(\top), -1\bigr)]$. The quantitative monitor's reward-specification pairs are defined symmetrically. Intuitively, the task completion metric is defined as either the normalised Manhattan distance (scaled between $[0, 1]$ given the maximum Manhattan distance for an 8x6 grid is 12) to the goal tile or 0 upon entering a water drape.

\paragraph*{Sokoban} is an environment designed with the intent to evaluate avoiding side effects, requiring balancing reaching the goal tile with a penalty for moving the box into an irreversible position. 

On one hand, we must reach the goal, which can be achieved with the scaled BFS distance for the 6x6 grid. However, the agent should move to the goal position whilst causing minimal irreversible side effects, despite the fact that the box is an obstacle in a way that must be moved into a non-ideal position.

Therefore, we consider a fluent called the $wall\_penalty$, with the set of walls denoted as $\mathcal{W}$. First, let us define the four von-Neumann neighbours (four orthogonally adjacent cells) of the box's position $(x, y)$, $\: \mathcal{N}(x,y)=\{(x+1,y),\,(x-1,y),\,(x,y+1),\,(x,y-1)\}$. For each neighbour $p \in \mathcal{N}(x,y)$ we can determine whether it is a wall or not by $\delta(p)$:
$$
\delta(p)=
\begin{cases}
1,& p\in \mathcal{W},\\[4pt]
0,& p\notin \mathcal{W}.
\end{cases}
$$

A corner can be detected by two walls forming an orthogonal pair when one is horizontal and the other vertical: $\mathrm{corner}(x,y) =
\bigl[\delta(x-1,y)\land\delta(x,y-1)\bigr]\;\lor\;
\bigl[\delta(x-1,y)\land\delta(x,y+1)\bigr]\;\lor\;
\bigl[\delta(x+1,y)\land\delta(x,y-1)\bigr]\;\lor\;
\bigl[\delta(x+1,y)\land\delta(x,y+1)\bigr].$ Then, we can set the $wall\_penalty$ fluent as: $$ \operatorname{wall\_penalty}(x,y) = \begin{cases}
1, & \text{if }\mathrm{corner}(x,y)=\text{true},\\[6pt]
\dfrac12, & \text{if }\displaystyle\sum_{p\in\mathcal{N}(x,y)}\delta(p)=1,\\[10pt]
0, & \text{otherwise}.
\end{cases}
$$
In the Boolean form, the vase touching a single wall or being in a corner is treated equally as setting the $wall\_penalty$ to $\false$.

The quantitative reward specification can be defined as: $[ \bigl(\Eventually(\Always(reach\_goal)), 100\bigr), \bigl( \Always(\lnot wall\_penalty), 100\bigr), \\ \bigl(\eveAlwTrue, -1\bigr)]$. And the Boolean specification can be defined as: $[\bigl(\Eventually(reach\_goal), 100\bigr), \bigl(\Always(\lnot wall\_penalty), 100\bigr), \\ \bigl(\Always(\true), -1\bigr)]$.

\paragraph*{Conveyor Belt} is the task of rescuing a vase from breaking. The vase is pushable by the agent with Sokoban-style dynamics, and is at risk of breaking if it falls off a moving conveyor belt which it is initially placed on. The optimal behaviour is to secure the vase and not re-place it on the conveyor belt where it is at risk of breaking. 

We define the following fluents $vase\_broken$ (which occurs when the vase reaches a specific tile, falling off the conveyor belt), $vase\_off\_conveyor$, and $reach\_vase$ which gives the negative Manhattan distance, clamped between 0 to 1 using the maximum Manhattan distance of 12 for the 7x7 grid, defined as following:

$$
\begin{aligned}
reach_{\text{vase}}(x, y, x_{\text{vase}}, y_{\text{vase}}) = \\
&\hspace{-2cm}\max\!\Bigl(
     0,\,
     1 - \frac{|x - x_{\text{vase}}| + |y - y_{\text{vase}}|}
              {2\,(7 - 1)}
  \Bigr)
\end{aligned}
$$

The Boolean reward specification is given by: $[\bigl(\Eventually(vase\_off\_conveyor) \land \Always(\lnot vase\_broken), 100\bigr), \\ \bigl(\Always(\lnot vase\_broken), 100\bigr)]$. And the quantitative reward specification is extended with: $[\bigl(\Eventually(\Always(vase\_off\_conveyor)) \land \Eventually(\Always(reach\_vase)), 100 \bigr), \bigl((\Eventually(\Always(vase\_off\_conveyor)) \lor \Eventually(\Always(reach\_vase))) \land \Always(\lnot vase\_broken), 100 \bigr), \\ \bigl(\Always(\lnot vase\_broken), 100 \bigr)]$.

Conveyor Belt specifically is a challenging environment to learn as within the span of a few (4-5) timesteps, if the vase is not rescued from the conveyor, a large positive segment of the reward can never be attained. Moreover, the episode does not actually terminate when the vase is broken unlike others, as the original intent of the environment was to evaluate negative side-effects, which increases the learning difficulty.

\section{Algorithms}
\label{sec:appendix_hyperparams}
All experiments were run using Python 3.11 running CUDA 12.8 on Ubuntu Linux 24.04.2 LTS with 256 GiB of RAM, Intel Xeon Processor with 192 cores at 2.4 GHz clock speed, and a Nvidia Tesla A100 (80GB vRAM). The hyperparameters used for the algorithms are listed as defaults to efficiently allow evaluating all reward mechanisms under identical conditions, with specific hyperparameters and number of trials detailed in the following subsections. 

\subsection{Number of Trials}
Each environment was run $n$ times with $m$ number of episodes (in the case of Tabular Q-Learning) or training steps (in the case of PPO). This helps to reduce the impact of any non-deterministic behaviour (e.g. random exploration during training). For \textsc{Toy} and \textsc{Safety-Gridworlds}, we used a single (environmental) labelled seed to control the object layout of the environment, and random seeds based on number of runs for other environments. In Table~\ref{table:cw}, a policy is said to have converged to a suboptimal policy (denoted by $*$) when there is another policy trained that achieves a higher value for the metric of task completion which we use as our performance function.

\begin{table}[hbtp]
    \centering
    \caption{Number of trials used with environments trained via PPO}
    \label{tab:ppo_params}
    \resizebox{\linewidth}{!}{
    \begin{tabular}{@{}lrrrr@{}}
        \toprule
        \textbf{Environment} & \#\textbf{Runs} & \textbf{Steps / Update} & \#\textbf{Parallel Envs} & \textbf{Total Timesteps} \\ \midrule
        Acrobot         & 10 & 128  & 4 & 500\,000 \\
        Cartpole        & 10 & 128  & 4 & 500\,000 \\
        Mountain Car    & 10 & 128  & 4 & 500\,000 \\
        Pendulum        & 10 & 2048 & 1 & 1\,000\,000 \\
        Bipedal Walker  & 10 & 2048 & 1 & 500\,000 \\
        Lunar Lander    & 10 & 128  & 4 & 500\,000 \\ \bottomrule
    \end{tabular}
    }
\end{table}

\begin{table}[hbtp]
    \centering
    \caption{Number of trials used with environments trained via Tabular Q-Learning}
    \label{tab:q_learning_params}
    \resizebox{\linewidth}{!}{
    \begin{tabular}{@{}lrrr@{}}
        \toprule
        \textbf{Environment} & \#\textbf{Episodes} & \#\textbf{Runs} & \textbf{Max Steps / Episode} \\ \midrule
        Conveyor Belt      & 2\,000 & 500 & 100 \\
        Island Navigation  & 2\,000 & 500 & 100 \\
        Sokoban           & 2\,000  & 500 & 100 \\
        Cliff Walking      & 2\,000 & 500 & 100 \\
        Frozen Lake        & 2\,000 & 500 & 100 \\
        Taxi              & 2\,000  & 500 & 100 \\ \bottomrule
    \end{tabular}
    }
\end{table}

\subsection{Hyperparameters of Tabular Q-Learning}
For tabular Q-Learning, in the \textsc{toy} and \textsc{safety-gridworld} environments, we used the following hyperparameters:
\begin{table}[hbtp]
  \centering
  \caption{Hyper-parameters for $\epsilon$-greedy Q-Learning training}
  \label{tab:dqn-hparams}
  \begin{tabular}{ll}
    \toprule
    \textbf{Parameter} & \textbf{Value} \\
    \midrule
    Learning rate                   & 0.01 \\
    Initial exploration rate $\epsilon$ & 1.0 \\
    $\epsilon$ decay                & 0.9985 \\
    Minimum $\epsilon$              & 0.05 \\
    Discount factor $\gamma$        & 0.9 \\
    \bottomrule
  \end{tabular}
\end{table}

\subsection{Hyperparameters of Proximal Policy Optimization}
To handle continuous action and state spaces (as part of \textsc{Box2d} and \textsc{Classic} environments), alongside Tabular Q-Learning we also used PPO which also exhibits state-of-the-art performance within many environments. The specific implementations used were adapted from \textsc{CleanRL}.

\begin{table}[htbp]
  \centering
  \caption{Hyper-parameters for PPO in discrete-action environments}
  \label{tab:ppo-discrete-hparams}
  \begin{tabular}{ll}
    \toprule
    \textbf{Parameter} & \textbf{Value} \\
    \midrule
    Steps per rollout $n_{\text{steps}}$  & 128 \\
    Torch deterministic              & True \\
    CUDA enabled                     & True \\
    Convergence window size          & 100 \\
    Convergence threshold            & 0.1 \\
    Total timesteps                  & 500\,000 \\
    Learning rate                    & $2.5\times10^{-4}$ \\
    Number of parallel envs          & 4 \\
    LR annealing                     & True \\
    Discount factor $\gamma$         & 0.99 \\
    GAE $\lambda$                    & 0.95 \\
    Minibatches per update           & 4 \\
    Epochs per update                & 4 \\
    Normalise advantages             & True \\
    Clip coefficient                 & 0.2 \\
    Clip value-loss                  & True \\
    Entropy coefficient              & 0.01 \\
    Value-loss coefficient           & 0.5 \\
    Max.\ gradient norm              & 0.5 \\
    Target KL                        & None \\
    \bottomrule
  \end{tabular}
\end{table}

\begin{table}[htbp]
  \centering
  \caption{Hyper-parameters for PPO in continuous-action environments}
  \label{tab:ppo-continuous-hparams}
  \begin{tabular}{ll}
    \toprule
    \textbf{Parameter} & \textbf{Value} \\
    \midrule
    Steps per rollout $n_{\text{steps}}$  & 2\,048 \\
    Torch deterministic              & True \\
    CUDA enabled                     & True \\
    Convergence window size          & 100 \\
    Convergence threshold            & 0.1 \\
    Total timesteps                  & 1\,000\,000 \\
    Learning rate                    & $3\times10^{-4}$ \\
    Number of parallel envs          & 1 \\
    LR annealing                     & True \\
    Discount factor $\gamma$         & 0.99 \\
    GAE $\lambda$                    & 0.95 \\
    Minibatches per update           & 32 \\
    Epochs per update                & 10 \\
    Normalise advantages             & True \\
    Clip coefficient                 & 0.2 \\
    Clip value-loss                  & True \\
    Entropy coefficient              & 0.0 \\
    Value-loss coefficient           & 0.5 \\
    Max.\ gradient norm              & 0.5 \\
    Target KL                        & None \\
    \bottomrule
  \end{tabular}
\end{table}
\fi

\end{document}